\documentclass{article}
\usepackage[a4paper, total={6.5in,8.5in}]{geometry}

\usepackage{macros}
\usepackage{subcaption}
\usepackage{graphicx}
\usepackage{algorithm}
\usepackage{algpseudocode}
\usepackage{booktabs} 
\usepackage{multirow} 
\usepackage[numbers, sort, comma, square]{natbib}
\usepackage{authblk}
\usepackage{wrapfig}
\usepackage{appendix,minitoc,titletoc}

\newcommand\DoToC{%
  \startcontents
\hypersetup{colorlinks=true, linkcolor=pierCite}
  \printcontents{}{1}{\subsection*{\textbf{Table of contents}}}
  \vskip3pt\vskip5pt
}

\graphicspath{{figures/}}

\author{Piersilvio De Bartolomeis\footnote{These authors contributed equally.}~}
\author{Javier Abad$^*$}
\author{Konstantin Donhauser}
\author{Fanny Yang}
\affil{Department of Computer Science, ETH Zürich}

\title{Hidden yet quantifiable: A lower bound  for confounding strength using randomized trials}

\date{}

\begin{document}

\maketitle

\begin{abstract}
In the era of fast-paced precision medicine, observational studies play a major role in properly evaluating new treatments in clinical practice. Yet, unobserved confounding can significantly compromise causal conclusions drawn from non-randomized data. We propose a novel strategy that leverages randomized trials to quantify unobserved confounding. First, we design a statistical test to detect unobserved confounding above a certain strength. Then, we use the test to estimate an asymptotically valid lower bound on the unobserved confounding strength. We evaluate the power and validity of our statistical test on several synthetic and semi-synthetic datasets. Further, we show how our lower bound can correctly identify the absence and presence of unobserved confounding in a real-world example.\footnote{See our GitHub repository for the source code: \url{https://github.com/jaabmar/confounder-lower-bound}.}

\end{abstract}

\section{Introduction}
Monitoring the performance of a newly approved treatment is crucial, a process commonly referred to as \emph{post-marketing surveillance}~\citep{vlahovic2011postmarketing}. Nowadays, the U.S. Food and Drug Administration promotes the integration of observational data in this process to address the shortcomings of randomized evidence~\citep{platt2018fda,klonoff2020new}. This strategy is essential for validating personalized treatments, like immunotherapy for certain types of cancer, where randomized evidence is scarce, and treatment costs are substantial~\citep{hayden2016gene,goetz2018personalized}.

 Yet, unobserved confounding can significantly compromise causal conclusions drawn from observational data.  To tackle this issue, sensitivity analysis has been the prevalent paradigm since its conception by~\citet{cornfield59}.
 This field studies how a specific strength of unobserved confounding affects causal conclusions and introduces the concept of a \emph{critical value}~\citep{vanderweele2017sensitivity, jin2023sensitivity}, i.e. the minimum strength unobserved confounders would need to have to explain away the estimated treatment effect. However, critical values are solely based on observational data and can differ substantially from the \emph{true confounding strength}. As a result, epidemiologists often rely on heuristic judgments to decide whether an observational study is flawed.

Estimating the true confounding strength is infeasible without further assumptions.
Yet, once a treatment gains approval, we may have access to a randomized trial that allows for more effective strategies to address unobserved confounding. A recent line of works proposes to combine the estimators from randomized and observational data, e.g. see \citet{colnet2020causal,brantner2023methods} for a survey. However, these methods crucially rely on some prior knowledge of the confounding bias structure, that is not always available in practice. 
 
\begin{figure*}
    \centering 
    \includegraphics[width=\textwidth]{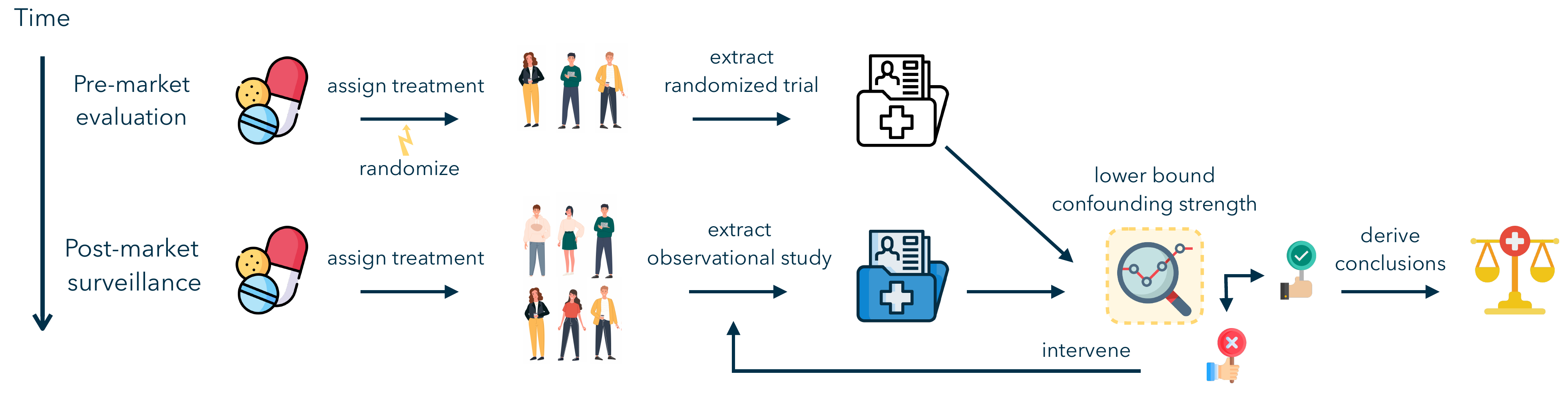} 
    \caption{An illustrative example of the drug regulatory process: our lower bound allows taking proactive measures to address the unobserved confounding problem.} 
    \label{fig:motivating_example} 
\end{figure*}

We propose an alternative strategy to leverage randomized trials, that is, to test and quantify the true confounding strength. In particular, if strong confounding is detected, epidemiologists can take proactive measures to correct it. Most directly, they can identify and incorporate relevant covariates into the study design if they were initially overlooked~\citep{dreyer2018advancing}. On the other hand, if small confounding is detected, epidemiologists can continue their analysis (see~\Cref{fig:motivating_example} for an illustration of the pipeline). More concretely, our contributions are as follows.
\begin{itemize}
\vspace{-2mm}
\item In~\Cref{sec:test}, we introduce the first statistical test to detect unobserved confounding above a certain strength. Further, we show how the test can be used to estimate an asymptotically valid lower bound on the true confounding strength.
\item In \Cref{sec:syn_exp}, we evaluate the finite-sample validity and power of our test on several synthetic and semi-synthetic datasets. 
\item In~\Cref{sec:whi_exp},
we showcase through a real-world example how our approach leads to conclusions that align with established epidemiological knowledge.
\end{itemize}

\subsection{Related work}
\label{sec:related_work}
Our approach is closely related to a line of work that proposes statistical tests for the \emph{presence} of unobserved confounding. In particular, several works leverage randomized trials to detect unobserved confounding. These tests check for significant differences between average treatment effect estimates obtained from randomized and observational data \citep{viele2014use,yangelastic, morucci2023double,hussain2022falsification}. More sophisticated approaches also test for differences in conditional average treatment effect estimates \citep{hussain2023falsification} and account for right-censored outcomes~\citep{demirel2024benchmarking}.

Similarly, other works have designed statistical tests using instrumental variables and negative control outcomes instead of randomized trials~\citep{lipsitch2010negative, de2014testing, donald2014testing,sofer2016negative}. Additionally, multiple observational studies can be leveraged to test conditional independences and detect unobserved confounding~\citep{karlsson2023detecting}.

In contrast to our test, these works have a significant limitation: they cannot quantify the true confounding strength. Even in infinite samples, they reject observational studies with negligible confounding. In real-world settings, where some degree of confounding will likely be present, testing for the absence of unobserved confounding can be too restrictive. 

Finally, another line of works proposes calibrating the value of confounding strength using only observational data~\citep{hsu2013calibrating,veitch2020sense}. However, the true confounding
 strength can be arbitrarily different from the calibrated strength, and there is no theoretical result for how these two quantities are related, even with infinite samples.

\section{Setting and notation}
\label{sec:setting}
We have access to data from a randomized trial~(rct) and an observational study~(os), which come from an underlying distribution $\pfull^\diamond$ 
over $\left(X, U, Y(0), Y(1),Y, T\right)$, for $\diamond \in  \{\rct,\obs\}$. Here, $\left(X, U\right) \in \RR^\xdim \times \RR^\udim$ is a vector of confounders, $\left(Y(0), Y(1)\right)$ are real-valued bounded potential outcomes, $Y \in \RR$ is the observed outcome, and $T \in \{0,1\}$ is a binary treatment 
indicator. However, the confounder $U$ and the potential outcomes are never observed, that is, we can only sample from the distributions $\prct \defeq \margin(\pfullrct)$ and $\pobs \defeq \margin(\pfullobs)$, where $\margin(\pfull)$ denotes the marginal distribution of $(X,Y,T)$ under $\pfull$.

\begin{wrapfigure}{r}{0.4\textwidth} 
\vspace{-5mm}
\begin{center}
    \includegraphics[width=0.4\textwidth]{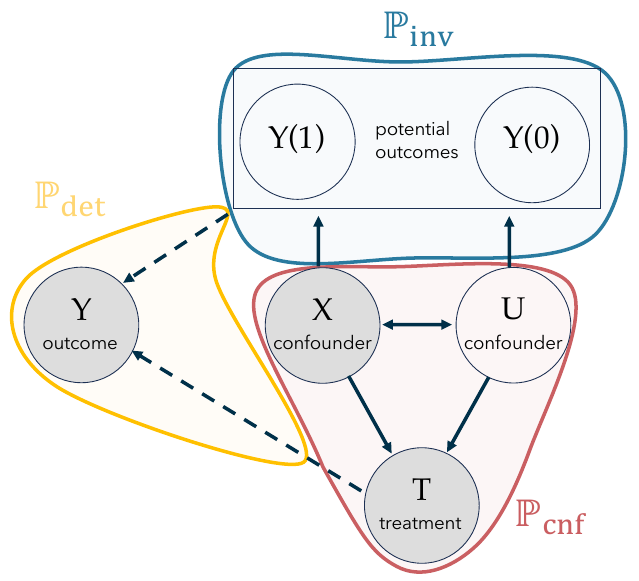}
  \end{center}
  \caption{Graphical model that captures the Neyman-Rubin potential outcome framework with unobserved confounder $U$. $\pinv$ is the causal mechanism that does not change between the randomized trial and the observational study, while $\pconf$ changes across studies. 
  For the randomized trial, we assume there is no arrow from the confounders ($X, U$) to the treatment indicator $T$ due to its internal validity. 
  Observed variables are colored in shades of grey.}
          \label{fig:dag}
\vspace{-2mm}
\end{wrapfigure}

We assume that  we can factorize the full distribution as follows for $\rct$ and $\obs$
\begin{align}
\pfull^{\diamond} &= \underbrace{\p_{Y  \mid Y(1),Y(0),T}}_{\defeq \pdet}~\underbrace{\p_{Y(1),Y(0)  \mid X,U}}_{\defeq \pinv}~\underbrace{\p^{\diamond}_{ X,T, U}}_{\defeq \pconf^{\diamond}},\label{eq:factor}
\end{align}
where $\pdet$ is given by $Y=Y(T)$, $\pinv$ is invariant across studies, and $\pconf^\diamond$ differs for $\diamond \in \{\rct,\obs\}$. This factorization captures the essence of the potential outcome framework, where $Y(1)$ and $Y(0)$ do not depend on $T$
while being more general. In particular, it allows for shifts in the marginal distribution of the observed and unobserved confounders. We illustrate the corresponding graphical model in \Cref{fig:dag}. While there have been many attempts to unify potential outcomes and graphical models, most prominently Single World Intervention Graphs~\citep{richardson2013single}, here the graph is purely used for illustrative purposes and does not play a formal role in our analysis.

We now introduce three additional assumptions required for the validity of our statistical test and the resulting lower bound.
First, we require transportability of the conditional average treatment effect (CATE).

\begin{assumption}[Transportability] 
\label{asm:transportability}
The conditional average treatment effect remains invariant across studies, that is
\begin{align*}
\EE_{\pfullobs}\left[Y(1) -Y(0)  \mid X\right]
= \EE_{\pfullrct}\left[Y(1) -Y(0)  \mid X\right].
\end{align*}
\end{assumption}
This property is standard for generalizing the findings of randomized trials to another population~\citep{o2014generalizing,colnet2020causal,degtiar2023review}, and is a weaker assumption than  ignorability of study selection~\citep{sugden1984ignorable, hotz2005predicting}
or sample ignorability of treatment effects~\citep{kern2016assessing}. 

Second, we assume that the randomized trial is internally valid.

\begin{assumption}[Internal validity] 
\label{asm:internalvalid}
The treatment is assigned independent of the covariates and the potential outcomes, that is, 
\begin{align*}
\pconfrct  =\prct_{T} \prct_{X, U}, \quad \mathrm{with} \quad \prct_T(T=1) = \pi \in (0,1). 
\end{align*}
\end{assumption} 
Internal validity holds by design in a completely randomized experiment, allowing for an unbiased estimation of the treatment effect. Observational studies, on the other hand, can have arbitrary confounding structures reflected in $\pconf$, i.e. $\pconfobs  =\pobs_{ T \vert X,U} ~\pobs_{X,U}$.

Finally, we assume that the population in the observational study includes the population in the trial.
\begin{assumption}[Support inclusion] 
\label{asm:nested_support}
The support of the randomized trial is included in the support of the observational study, i.e.
$$\supp(\prct_X) \subseteq \supp(\pobs_X).$$
\end{assumption}
 This assumption is strictly weaker than the positivity of trial participation~\citep{stuart2011use, hartman2015sample, andrews2017weighting, nie2021covariate, colnet2022reweighting}. It is also expected to hold in our setting, as it aligns with the design of observational studies by regulatory agencies for drug monitoring ~\citep{he2020clinical}, and particularly for post-marketing surveillance~\citep{franklin2019evaluating, schurman2019framework}.

\subsection{Sensitivity analysis}
\label{sec:sensana}
Sensitivity analysis is commonly used to account for unobserved confounding in observational data. In particular, this approach estimates an interval for the treatment effect that depends on an assumed
\emph{confounding strength} $\confvalue$ of $\pcnfos$. Throughout the paper,  we define the confounding strength using the widely accepted marginal sensitivity model~\citep{tan2006distributional}. 

More formally,
we assume that $\pcnfos$ belongs to the set $\msm(\confvalue)$ of distributions
that have bounded odds ratio,   
\begin{align*}
	&\msm(\confvalue) \defeq  \Big  \{ \pconf: \frac{1}{\confvalue} \leq   \frac{\pconf(T=1 \mid X, U)}{\pconf(T=0 \mid X, U)} / \frac{\pconf(T=1 \mid X)}{\pconf(T=0 \mid X)}  \leq \confvalue, \;\; \mathrm{a.s.}   \Big \}.
\end{align*}
Under this notion of confounding strength,  we can define a set of full distributions $\pfulltilde$ that are compatible with the marginal distribution of the observational study $\pobs$ and have a bounded odds ratio. 
 \begin{definition}[Marginal sensitivity set]
\label{def:msm} 
Given a distribution $\pobs$ over $(X,Y,T)$ and a confounding strength $\confvalue \geq 1$, we define the set $\msm(\pobs, \confvalue)$ of distributions $\pfulltilde$, as
\begin{align}
\msm(\mathbb \pobs, \confvalue) \defeq   
\{ \pfulltilde = \pdet \pinvtilde \pcnftilde: \pcnftilde \in \msm(\confvalue) 
 \quad \mathrm{and} \quad\margin(\pdet\pinvtilde\pcnftilde) = \pobs \nonumber
\}.
\end{align}
 \end{definition}
In other words, this set contains all the full distributions that could have induced the marginal distribution of the observational study $\pobs$.
Further, since the marginal sensitivity set contains $\pfullobs$ if $\confvalue$ is well-specified, we can partially identify the (conditional) treatment effect as follows. 
\begin{definition}[Sensitivity bounds] 
\label{def:cate_sensitivity}
We define the conditional average treatment effect~(CATE) as 
\begin{align*}
\yone (X, \pfull) &\defeq \EE_{\pfull}\left[ Y(1) - Y(0) \mid X \right], 
\end{align*}
and the upper and lower bounds on CATE within the marginal sensitivity set as 
\begin{align*}
\yone^+_\confvalue(X) \defeq \sup _{\pfulltilde \in \msm(\pobs, \confvalue)} \yone(X,\pfulltilde),\quad 
\yone^{-}_\confvalue(X) \defeq \inf_{\pfulltilde \in \msm(\pobs, \confvalue)} \yone(X,\pfulltilde).
\end{align*}
Further, we define the average treatment effect (ATE) over a marginal distribution $\p_X$ that can differ from the marginal in $\pfull$ as
	$$\yone(\p_X,\pfull) \defeq \EE_{\p_X}\left[\cate(X,\pfull)\right],$$
and the upper and lower bounds on ATE as
\begin{align*}
\yone^+_\confvalue(\p_X) \defeq \EE_{\p_X}\left[\cate^+_\confvalue(X)\right],  \quad 	\yone^-_\confvalue(\p_X) \defeq \EE_{\p_X}\left[\cate^-_\confvalue(X)\right].
\end{align*}
\end{definition}

Above, we do a slight abuse of notation by defining $\mu$ as both a function and a real number, depending on its argument.
Several estimators have recently emerged in the literature for the CATE bounds~\citep{kallus2019interval,jesson2021quantifying,oprescu2023b}  and for the ATE bounds~\citep{zhao2019sensitivity,dorn2021doubly,dorn2022sharp}. We will leverage these estimators to construct a statistical test that detects unobserved confounding above a certain strength.

\section{Methodology}\label{sec:test}

We would like to test whether the unobserved full distribution $\pfullobs$, which marginalizes to $\pobs$, has confounding strength at most $\confvalue$. This is captured by the following null hypothesis
$$\hnull(\confvalue): \pfullobs \in \msm(\pobs,\confvalue).$$ 
Note that in the special case where $\confvalue=1$,  
the problem reduces to testing whether there are no unobserved confounders, i.e. $(Y(1),Y(0)) \indep T \mid X$ under $\pobs$. We refer to this case, which has been recently studied in the literature~(see~\Cref{sec:related_work}), as \emph{binary testing} for unobserved confounding.

In real-world scenarios, binary tests can be overly stringent, as they invalidate an observational study even if the unobserved
confounding strength is negligible. To overcome this limitation, we propose the first test, to the best of our knowledge, for the general case where $\confvalue$ is greater than one. In particular, underlying our testing procedure is a simple observation that follows from the
sensitivity analysis bounds: When the null hypothesis is true for some confounding strength $\Gamma$, the average treatment effect under some target population should fall between the valid upper and lower bounds constructed from the observational study. \begin{lemma}
\label{lemma:test_1} For any $\pfull$ which satisfies transportability, i.e. 
$
\yone(X,\pfull) =\yone(X,\pfullobs),
$
and any $\p_X$ which satisfies support inclusion, i.e. $\supp(\p_X) \subseteq \supp(\pobs_X)$, it holds that
\begin{equation*}
\pfullobs \in \msm(\pobs,\confvalue) \implies \yone(\p_X, \pfull) \in [\yone_\confvalue^-(\p_X), \yone_\confvalue^+(\p_X)].  
\end{equation*}
\end{lemma}\begin{proof}
First, note how $\yone(X, \pfull) \in [\yone_\confvalue^-(X), \yone_\confvalue^+(X)]$ for all $X \in \supp(\pobs_X)$ when the null hypothesis $\hnull(\confvalue)$ is true, due to the transportability assumption and the definition of CATE sensitivity bounds. 
The result then follows by taking expectations with respect to the corresponding marginals $\p_X$ on both sides.
\end{proof}

\subsection{Statistical tests for $\hnull(\confvalue)$}
In what follows, we have access to a randomized trial $\datarct = \{(X_i,Y_i,T_i)\}_{i=1}^{\nrct}$ sampled i.i.d. from the distribution $\prct$, and an observational study $\dataobs = \{(X_i,Y_i,T_i)\}_{i=1}^{\nobs}$, sampled i.i.d. from the distribution $\pobs$. We first propose estimates for the average treatment effect under two target populations. Then, we leverage these estimates together with the sensitivity bounds to design an asymptotically valid statistical test at significance level $\alpha$. Finally, we show how such a test can be used to establish an asymptotically valid lower bound on the unobserved confounding strength. 

\label{sec:test_presentetion}

\paragraph{Estimating the ATE} We discuss here how the average treatment effect can be estimated using data from the randomized trial. First, we define a target population $\ptar_X$ to estimate the ATE. 
Then, the following lemma shows how the choice of $\prct_X$ and $\pobsres_X \defeq \pobs_X \mid X \in \supp(\prct)$ allows us to identify ATE  using data sampled from the randomized trial marginal distribution $\prct$.

\begin{lemma}
\label{lemma:identify_ate}
For $\diamond \in \{ \rct, \obsres\}$, under Assumptions~\ref{asm:transportability}, \ref{asm:internalvalid} and     
\ref{asm:nested_support}, we have
\begin{align*}
\yone(\ptar_X, \pfulltar) &= \EE_{\prct} \left[ Y\left(\frac{T}{\pi} - \frac{1-T}{1-\pi}  \right)  w(X) \right], \quad \mathrm{where}\quad   w(X) \defeq \frac{\ptar(X)}{\prct(X)}.
\end{align*}
\end{lemma}
\Cref{lemma:identify_ate} is a well-known result in the transportability literature  \citep{cole2010generalizing,colnet2023risk}. It establishes that when the distribution shift between $\prct_X$ and $\ptar_X$ can be corrected, we can identify and estimate the ATE under $\ptar_X$. 

\paragraph{Estimating the sensitivity interval} Next, we discuss how $\yone_\confvalue^-(\ptar_X)$ and $\yone_\confvalue^+(\ptar_X)$ can be estimated using data from both the observational study and the target population $\ptar_X$. Here, the approach varies based on the target population.
\begin{itemize}[leftmargin=*]
	\item For $\ptar_X = \prct_X$, we estimate the CATE sensitivity bounds from observational data and average them over the target population. Specifically, we use the B-Learner \citep{oprescu2023b} to estimate the sensitivity analysis bounds.
	\item For $\ptar_X = \pobsres_X$, we have two options: either estimate the CATE sensitivity analysis bounds and average them, or directly estimate the ATE sensitivity analysis bounds over the target population.  In our experiments, we directly estimate the ATE sensitivity analysis bounds using either the DVDS \citep{dorn2021doubly} or the QB estimator \citep{dorn2022sharp}.
\end{itemize}   

These methods yield estimates that are valid, sharp, and efficient under more general conditions than other existing methods. Nevertheless, our testing procedure is agnostic to the choice of the sensitivity analysis bound estimator, allowing for various options to be adopted.

\paragraph{Two statistical tests} We outline our testing procedure in \Cref{algo:test}, which can be instantiated for the target populations $\rct$ and $\obsres$. This results in two statistical tests, $\test_{\rct}$ and $\test_{\obsres}$, for the null hypothesis $\hnull(\confvalue)$. The following proposition confirms their asymptotic validity.
\begin{prop}[Validity of the test]
\label{prop:test} Fix $\confvalue \in [1,\infty)$ and $\alpha \in (0,1)$, and let $\test_\diamond(\confvalue,\alpha)$ be the test defined in~\Cref{algo:test}. Under Assumptions~\ref{asm:transportability}--\ref{asm:nested_support}, the setting described in~\Cref{sec:setting}, we have, for $\hnull(\confvalue)$, 
 
(i) If it holds that $  \underset{\nrct,\nobs \to \infty}{\lim} \nrct/\nobs=  0$ and the estimators of
 the CATE sensitivity analysis bounds satisfy	
$$\|\cate_\confvalue^{\pm} - \catehat_\confvalue^{\pm}\|_{L^2(\prct)} = O_{\pobs}(\nobs^{-1/2}),$$
$\test_\rct(\confvalue, \alpha)$ is an asymptotically valid level-$\alpha$ test.
	
(ii) If it holds that $\underset{\nrct,\nobs \to \infty}{\lim}\nrct/(\nrct+\nobs)=\rho \in (0,1)$, $\EE_{\prct}[w(X)^2] < \infty$, and $\yonehat_\confvalue^+$ and $\yonehat_\confvalue^-$ are consistent estimators of the ATE sensitivity analysis bounds that satisfy
$$
 \sqrt{\nobs}\bigl(\yonehat_\confvalue^+ - \yone_\confvalue^+\bigr)\xrightarrow{\mathcal{D}} \gauss(0,(\sigma_\confvalue^+)^2), \qquad \sqrt{\nobs}\bigl(\yonehat_\confvalue^- - \meanlb\bigr)\xrightarrow{\mathcal{D}} \gauss(0,(\sigma_\confvalue^-)^2), 
$$
$\test_{\obsres}(\confvalue, \alpha)$ is an asymptotically valid level-$\alpha$ test.
\end{prop} 
We provide a complete proof in~\Cref{apx:proof_test}. Notably, Assumption (ii), together with the fixed limiting proportions of the randomized trial and observational study, is relatively mild and expected to hold for various estimators; for instance, it can be satisfied by the DVDS estimator~\citep{dorn2021doubly}. On the other hand, Assumption (i) is stronger and generally only expected to hold when $\nobs\gg\nrct$.

In essence, we propose two tests that work under different assumptions: $\test_{\rct}$ relies on a consistent estimate of the CATE sensitivity analysis bounds, while $\test_{\obsres}$ relies on directly estimating the ATE sensitivity analysis bounds together with the importance weights $w(X)$. The importance weights can be identified when the observational study and the randomized trial adhere to a nested trial design~\citep{olschewski1985comprehensive, olschewski1992analysis, choudhry2017randomized}. See \Cref{apx:nested_design} for a discussion on how the importance weights are estimated in this setting.

\paragraph{Advantages of each test} The test $\test_{\obsres}$ can be advantageous when CATE estimation is challenging (e.g. when the outcomes are binary and the classes are imbalanced or when the observational study has a limited sample size), but the weights $w(X)$ can be identified, and vice versa for the test $\test_{\rct}$. In addition, $\test_{\obsres}$ can benefit from large observational studies as the variances $(\sigmahat_\confvalue^-)^2$ and $(\sigmahat_\confvalue^+)^2$ vanish for large $\nobs$.
\vspace{-2mm}
\begin{algorithm*}[t]
\caption{Statistical test for detecting unobserved confounding}
\begin{algorithmic}[1]
\State \textbf{Input:}  $\diamond \in \{\rct,\obsres\}$, $\datarct,\dataobs$, significance level $\alpha$, confounding strength $\confvalue$.

\State Estimate $\yone(\ptar, \pfulltar)$ using the randomized trial dataset:  \begin{align*}\yonehat &=  \frac{1}{\nrct}\sum_{\substack{(X_i,T_i,Y_i)  \in \datarct}} Y_i\left(\frac{T_i}{\pi} - \frac{1-T_i}{1-\pi}  \right)w(X_i), 
\quad \sigmahat^2 = \widehat{\var}_{\prct}[\yonehat].
\end{align*} 

\State Estimate the sensitivity analysis bounds $\yonehat^{-}_\confvalue(X)$ and $\yonehat^{+}_\confvalue(X)$ using the observational study dataset, and average over the target population $\ptar$:$$
\yonehat_\confvalue^+ =
  \hat{\EE}_{\ptar_X}[ \yonehat^{+}_\confvalue(X)], \quad (\sigmahat^+_\confvalue)^2 = \widehat{\var}_{\ptar_X}[\yonehat_\confvalue^+], \quad   \yonehat_\confvalue^- = \hat{\EE}_{\ptar_X}[ \yonehat^{-}_\confvalue(X)], \quad (\sigmahat^-_\confvalue)^2 = \widehat{\var}_{\ptar_X}[\yonehat_\confvalue^-],
$$
where $\hat{\EE}[\cdot]$ denotes the empirical mean, and $\widehat{\var}[\cdot]$ denotes a consistent estimator of the sampling variance. 
\State Compute the test statistics: 
\begin{align*}  &\hat{T}_{\confvalue}^+= \frac{\yonehat^{+}_\confvalue - \yonehat}{\sigmahat^+_{\diamond}} , \quad \mathrm{where} \quad  \sigmahat^+_{\rct} =\sqrt{\varubhat +\sigmahat^2 + 2 \sigmahat_\confvalue^+ \sigmahat } \quad \mathrm{and}\quad \sigmahat^+_{\obsres} =\sqrt{\varubhat +\sigmahat^2},
\\\ &\hat{T}_{\confvalue}^- =  \frac{ \yonehat- \yonehat^{-}_\confvalue}{\sigmahat^-_{\diamond}},\quad \mathrm{where} \quad  \sigmahat^-_{\rct} =\sqrt{\varlbhat +\sigmahat^2 + 2 \sigmahat_\confvalue^- \sigmahat } \quad \mathrm{and}\quad \sigmahat^-_{\obsres} =\sqrt{\varlbhat +\sigmahat^2}.\end{align*}

\State \textbf{Output:} $\test_\diamond(\confvalue,\alpha) =  \indi \{\min\left(\hat{T}_\confvalue^+, \hat{T}_\confvalue^- \right) < z_{\alpha/2} \} $, where $z_\alpha$ is the $\alpha$-quantile of the standard normal. 
\end{algorithmic}
\label{algo:test}

\end{algorithm*}
 
\subsection{A lower bound on unobserved confounding strength}
\label{sec:lower_bound}
The statistical test described in the previous section raises a question about what level of confounding strength is reasonable to test. 
Ideally, epidemiologists would like to estimate the confounding strength instead of conducting a test. However, this is infeasible unless the support of $\prct_X$ and $\pobs_X$ are the same. 

A practical alternative is to estimate a lower bound on the true unobserved confounding strength defined as
$$
\Gammatrue \defeq  \inf \{ \confvalue : \pfullobs \in \msm(\pobs,\confvalue)\}.
$$
Given an observational study and a randomized trial, we aim to find a quantity that, with high probability, is a lower bound for the true confounding strength $\Gammatrue$. 
Without loss of generality, we fix the test $\test_{\rct}$ and recall that $\test_{\rct}(\confvalue, \alpha) $ is a deterministic function given the data\footnote{When bootstrap is used to estimate the variance we fix the bootstrap bags for all $\confvalue$.}. Hence, we obtain a lower bound for a fixed significance level $\alpha$ by computing
 \begin{equation}
 \label{eq:lb1}
     \gammalb = \inf_{\Gamma} \{\Gamma: \test_{\rct}(\confvalue,\alpha) =0 \},
 \end{equation}
that is, in words, the smallest $\Gamma$ such that the test accepts the null hypothesis. In practice, we compute $\gammalb$ with a grid search over values of $\confvalue$ starting from $1$ until the first test acceptance.

We show in the following proposition that $\gammalb$ is a valid lower bound for $\Gammatrue$. 
\begin{prop}
Let $\gammalb$ be as in Equation~\eqref{eq:lb1} for a fixed level $\alpha$.
Then, under the setting described in~\Cref{sec:setting}, and the same assumptions as in Proposition~\ref{prop:test},  
$\gammalb$ is an asymptotically valid lower bound, 
$$\p (\gammalb \leq \Gammatrue) \geq 1-\alpha - o(1).$$
\end{prop}
\begin{proof}
Note that by definition of $\gammalb$, we have that
\begin{align*}
    \p (\gammalb > \Gammatrue) &= \p (\cap_{\Gamma \leq \Gammatrue} \{\test_{\rct}(\confvalue, \alpha) =1\}) \\
    &\leq \p(\test_{\rct} (\Gammatrue,\alpha) =1) \leq \alpha + o(1),
\end{align*}
where the last inequality follows from the asymptotic validity of the test in Proposition~\ref{prop:test}. 
\end{proof}

\section{Synthetic Experiments}
\label{sec:syn_exp}
In this section, we evaluate our two tests and the resulting lower bounds in finite-sample synthetic and semi-synthetic experiments. In particular, we fix the true unobserved confounding strength $\Gammatrue$ and conduct experiments varying the sample size and the invariant distribution $\pinv$. 

First, we postulate that, for a fixed $\Gammatrue$,
the tightness of 
the lower bound $\gammalb$ improves when the confounder $U$ is more informative about the potential outcomes $(Y(1),Y(0))$. In our experiments, we choose the correlation between 
the unobserved confounder and one of the potential outcomes as a proxy measure of information,
\begin{equation}  
\label{eq:correlation}
\rho_{u,y} = \frac{\operatorname{Cov}_{\pfullobs}[Y(1),U]}{\sigma_{Y(1)} \sigma_U}.
\end{equation}
Intuitively, the sensitivity analysis bounds are tight for a specific $\Gammatrue$ when $\pcnfos$ leads to a marginal distribution $\pobs$ that maximally biases the estimable ATE. This situation occurs, for instance, when patients experiencing smaller outcomes are assigned to the control group while those with larger outcomes are in the treatment group. In this case, the sensitivity analysis bounds must be sufficiently large to include the true ATE and remain valid. 
Such a scenario is only possible if $U$ is very informative of $Y(1)$, captured by a high correlation coefficient. 
Conversely, when the confounder carries little information about the potential outcomes, the sensitivity bounds are unnecessarily conservative, leading to low power of the test and hence looser $\gammalb$. More formally, \Cref{apx:limitations_msm} discusses two extreme cases. If the unobserved confounder is maximally informative about the potential outcomes, namely \(U=(Y(1),Y(0))\), then \(\gammalb\) converges to \(\Gammatrue\) in the infinite-sample limit. Conversely, if \(U \indep (Y(1),Y(0))\), then the confounder carries no information about the potential outcomes and \(\gammalb \) converges to $1$.

Second,  we study the behavior of the lower bound as the observational study sample size grows. In real-world situations, increasing the number of samples in a randomized trial is often constrained by the logistical challenges of conducting additional experiments. However, observational studies have the potential for continuous growth through electronic health records and insurance claims databases. In the context of postmarketing, ongoing monitoring enables the inclusion of data from newly exposed individuals. Therefore, we compare our two tests when the sample size of the observational study grows: our experiments show that $\test_{\obsres}$ has better statistical power when $\nobs$ is large.
\subsection{Datasets}
\label{sec:datasets}
\paragraph{Synthetic distribution} We first benchmark tests and respective lower bounds with a synthetic distribution similar to \citep{yadlowsky2022bounds,jin2023sensitivity}. Here, the propensity score, the true unobserved confounding strength $\Gammatrue$, and the correlation strength can be designed.

We choose the invariant $\pinv$ to be the following linear outcome model 
$$
Y(T)= (2 T-1) X+(2 T-1) + U +\epsilon, \quad \epsilon \sim \gauss(0,\sigma^2_Y).
$$
For the marginal distribution over $X,U$ in $\pconf^{\diamond}$ we generate an unobserved confounder $U \sim \operatorname{Unif}(0,1)$ for both study designs and draw the observed covariate according to
$$ \prct_X = \operatorname{Unif}(-1,1), \; \pobs_X =  \operatorname{Unif}(-2,2).$$
Further, for the observational distribution, we choose the conditional distribution of the treatment $T$ given $X,U$ to be a  Bernoulli, which satisfies the marginal sensitivity model with an odds ratio equal to $\Gammatrue$.
Specifically, we fix the marginal propensity score as $$\pconfobs(T=1 \mid X)=\logit\left(0.75X + 0.5\right),$$ and design the full propensity score  $\pconfobs(T=1 \mid X, U)$ such that it marginalizes to $\pconfobs(T=1 \mid X)$.
For the randomized control trial, we choose 
$\pi = \pconfrct (T=1| X,U) = 1/2$.
We refer the reader to~\Cref{apx:syn_exp} for complete experimental details.

\paragraph{Semi-synthetic datasets} We expand our benchmark using three real-world randomized trials: Hillstrom's MineThatData Email data~\citep{hillstrom2008}, the Tennessee STAR study~\citep{word1990state} and the VOTE dataset~\citep{gerber2008social}.
In contrast to the synthetic experiments, these datasets involve real outcome functions, though the treatment assignment is still controlled. 

We focus on Hillstrom's dataset for clarity of presentation, and we refer the reader to~\Cref{apx:additional_experiments} for experiments on the other datasets showing similar trends. \citet{hillstrom2008} focused on measuring the impact of an email campaign on the dollars spent by the recipients in the following two weeks. We first sample a small subset of the original trial, $D$, as our randomized trial, $\datarct$. We can then subsample multiple observational studies from $D \setminus \datarct$ sharing a fixed true confounding strength $\Gammatrue$, i.e. $\pconf$, but with a varying correlation between the hidden confounder $U$ and outcome $Y(1)$, i.e. $\pinv$. 

Let us denote $\Xall$ as the vector of all observed covariates. While we cannot intervene on $\pinv (Y(1), Y(0)|\Xall)$ as it is intrinsic to the dataset, we can generate multiple observational studies by partitioning $\Xall$ into unobserved $U$ and observed $X$ in different ways.
For a given partitioning $\Xall = (U,X)$, the resulting $\dataobs$ will have a specific  $\pinv (Y(1), Y(0)|U)$  and hence correlation coefficient $\rho_{u,y}$. With each choice of $U$, we enforce a propensity score $\pconfobs(T=1 \mid U)$ that satisfies $\msm(\pobs, \Gammatrue)$
by subsampling $D \setminus \datarct$. Finally, we remove $U$ to construct $\dataobs$. Our subsampling approach is a variation of the methods presented in~\citet{keith2023rct, gentzel2021and} (see further details in Appendix \ref{apx:semisyn_exp}). Finally, we enforce \Cref{asm:nested_support} by excluding urban zip codes from the support of the randomized trial.

\begin{figure*}
\vspace{-9mm}
    \centering
\includegraphics[scale=0.25]{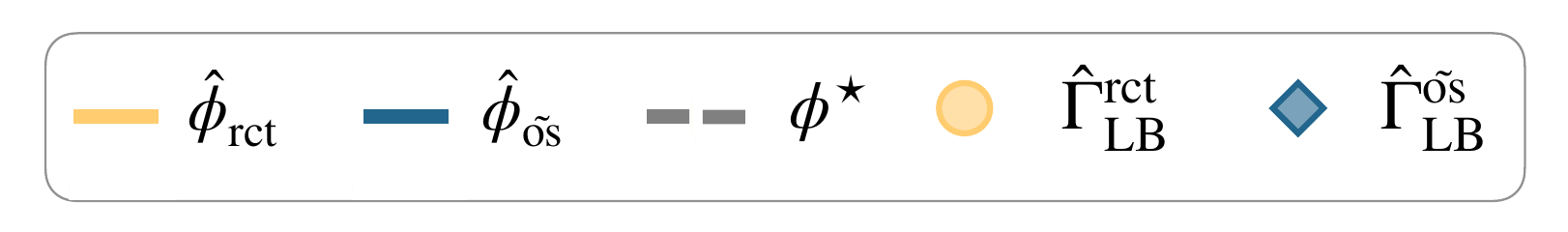}
      \\
    \begin{subfigure}[b]{0.24\textwidth}
        \centering
      \includegraphics[width=\textwidth]{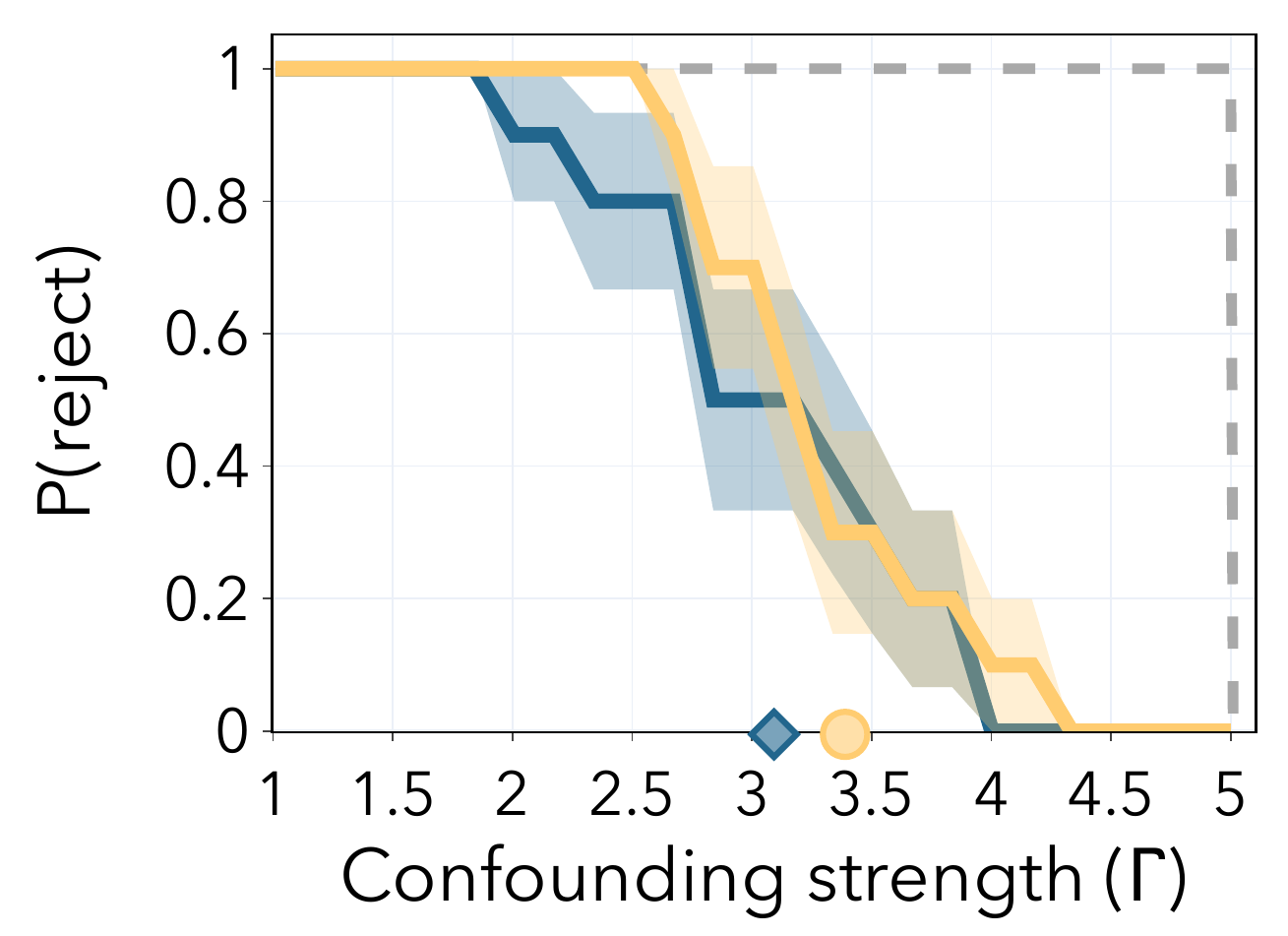}
        \caption{\footnotesize{Small sample (synthetic)}}
   \label{fig:synthetic_small_sample}
    \end{subfigure}
    \hfill
    \begin{subfigure}[b]{0.24\textwidth}
        \centering
\includegraphics[width=\textwidth]{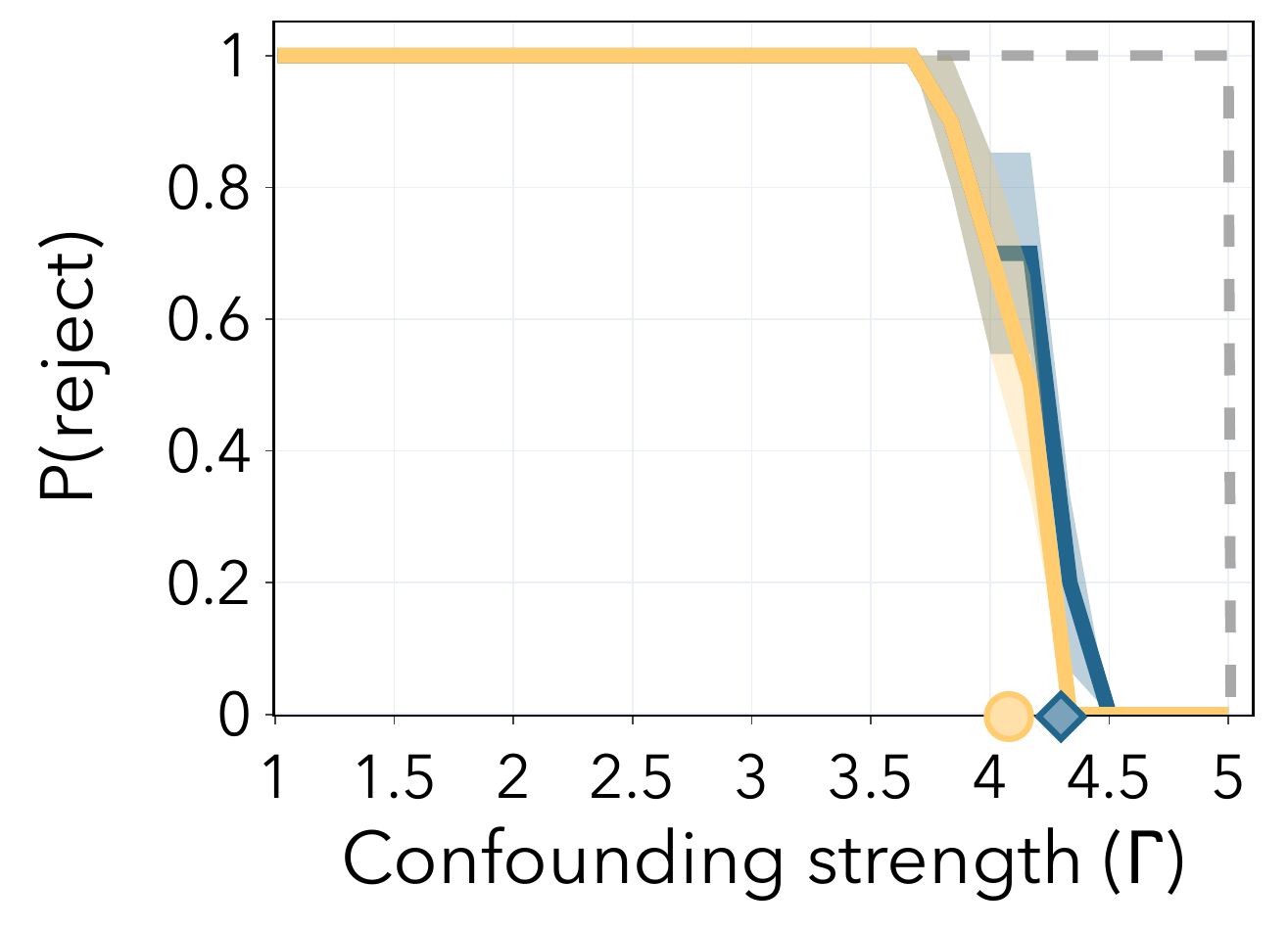}
        \caption{\footnotesize{Large sample (synthetic)}}
    \label{fig:synthetic_large_sample}
    \end{subfigure}
    \hfill
    \begin{subfigure}[b]{0.24\textwidth}
        \centering
   \includegraphics[width=\textwidth]{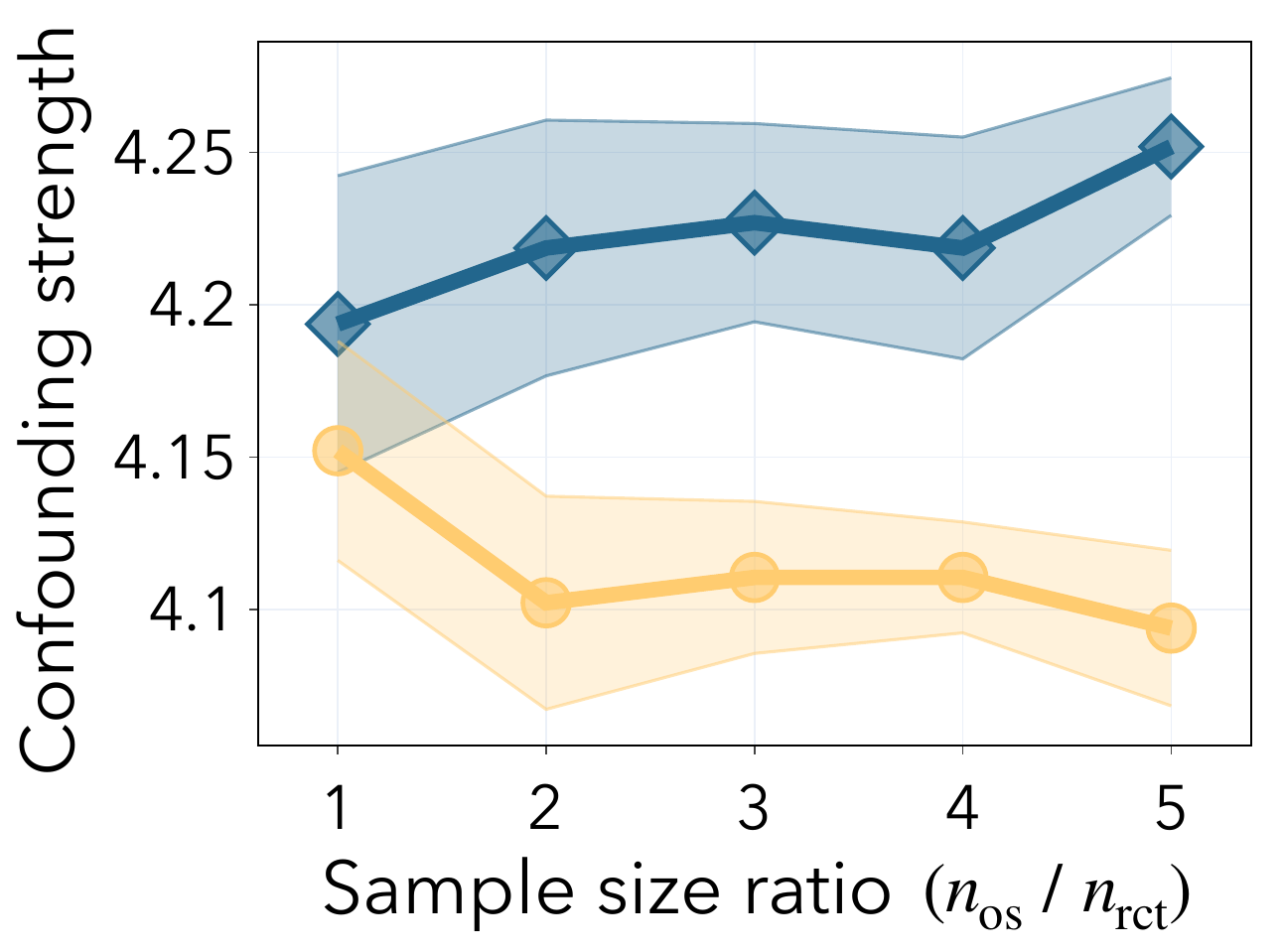}
        \caption{\footnotesize{Increasing $\nobs$ (synthetic)}}
        \label{fig:synthetic_obs_sample}
    \end{subfigure}
    \hfill
    \begin{subfigure}[b]{0.24\textwidth}
        \centering
    \includegraphics[width=\textwidth]{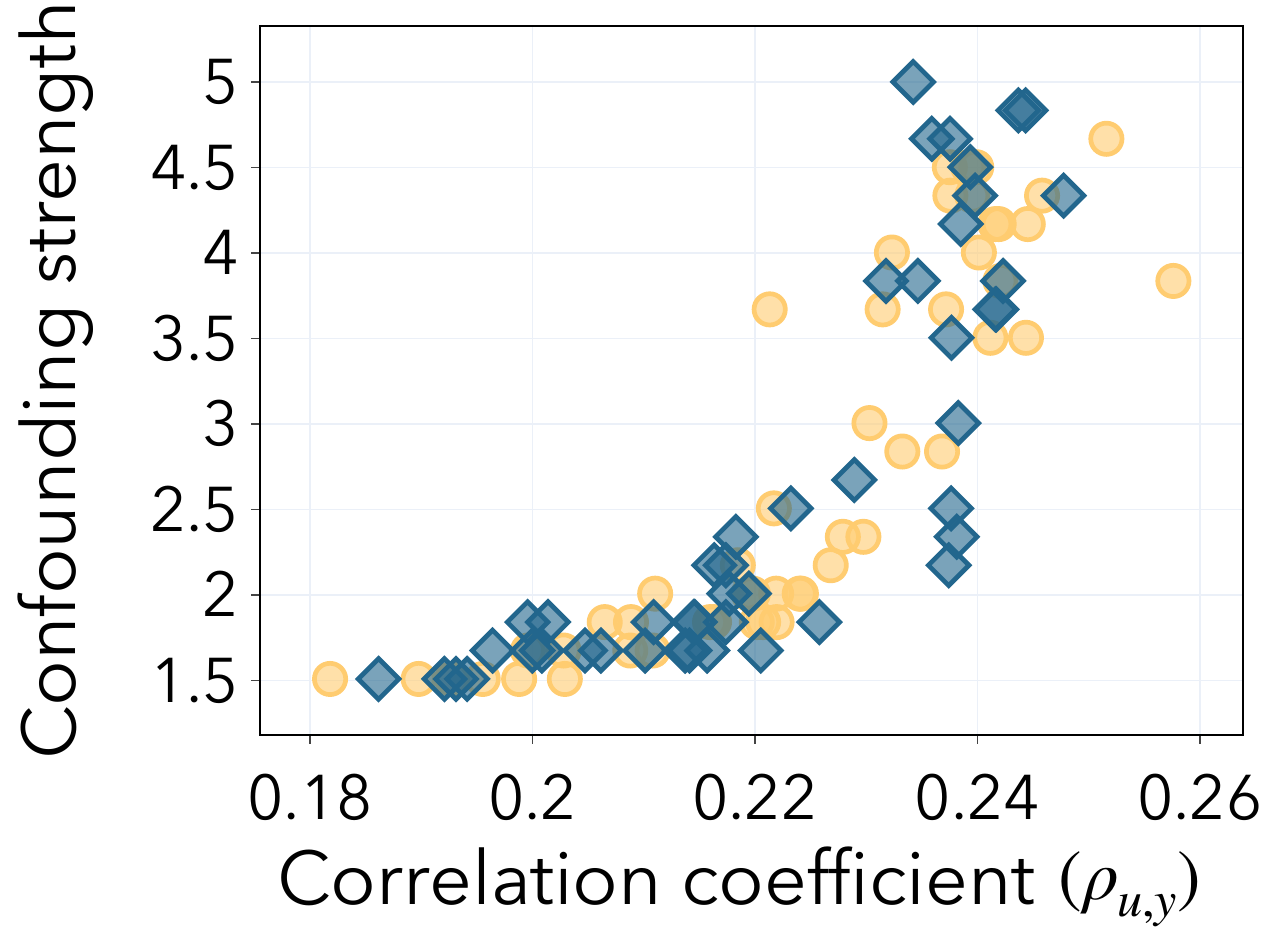}
        \caption{\footnotesize{$\rho_{u,y}$ effect (synthetic)} }
     \label{fig:synthetic_correlation}

    \end{subfigure}\\
    \begin{subfigure}[b]{0.24\textwidth}
        \centering
    \includegraphics[width=\textwidth]{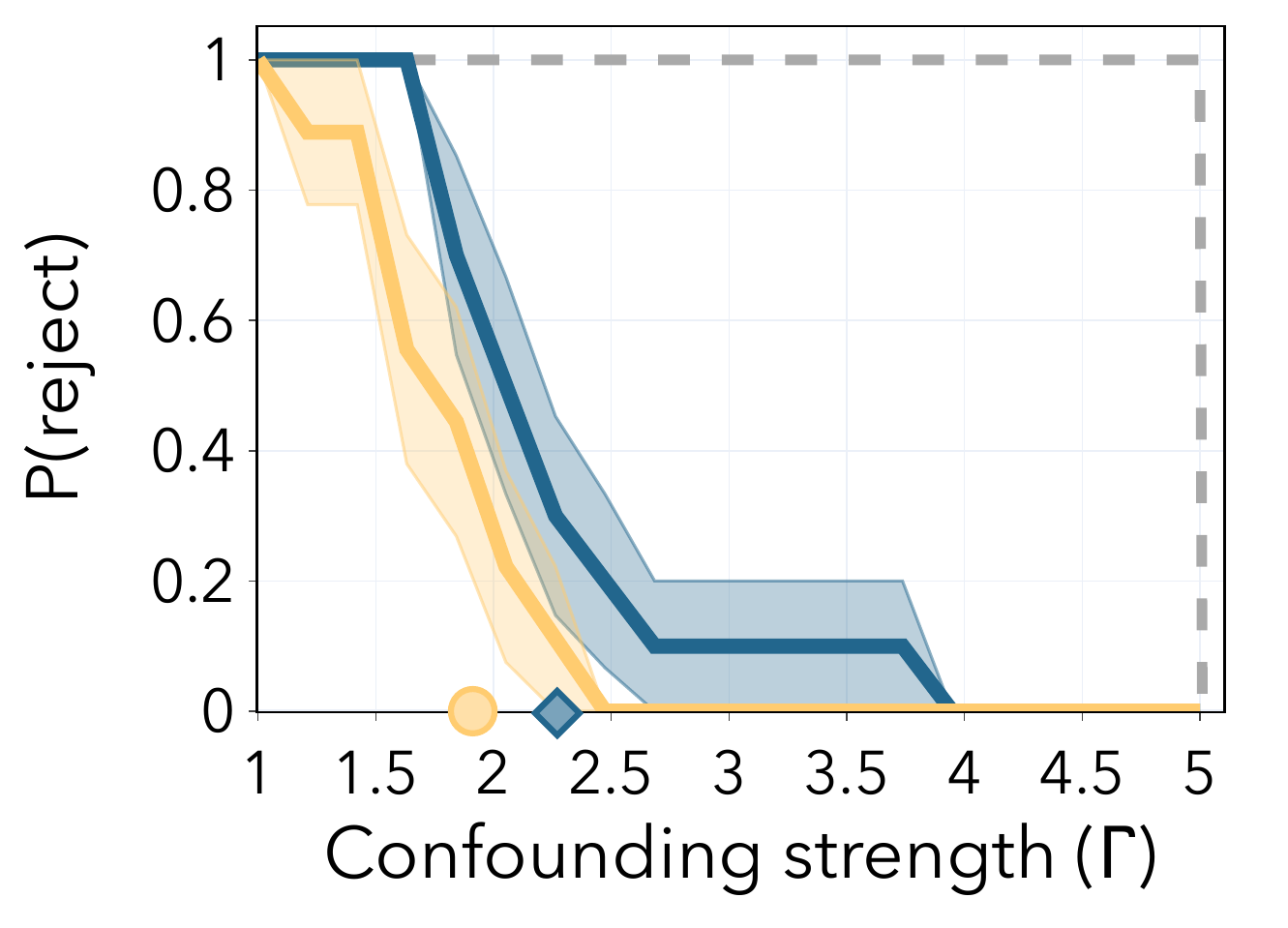}
        \caption{\footnotesize{Small sample (Hillstrom)} }
        \label{fig:hillstrom_small}
    \end{subfigure}
    \hfill
    \begin{subfigure}[b]{0.24\textwidth}
        \centering
\includegraphics[width=\textwidth]{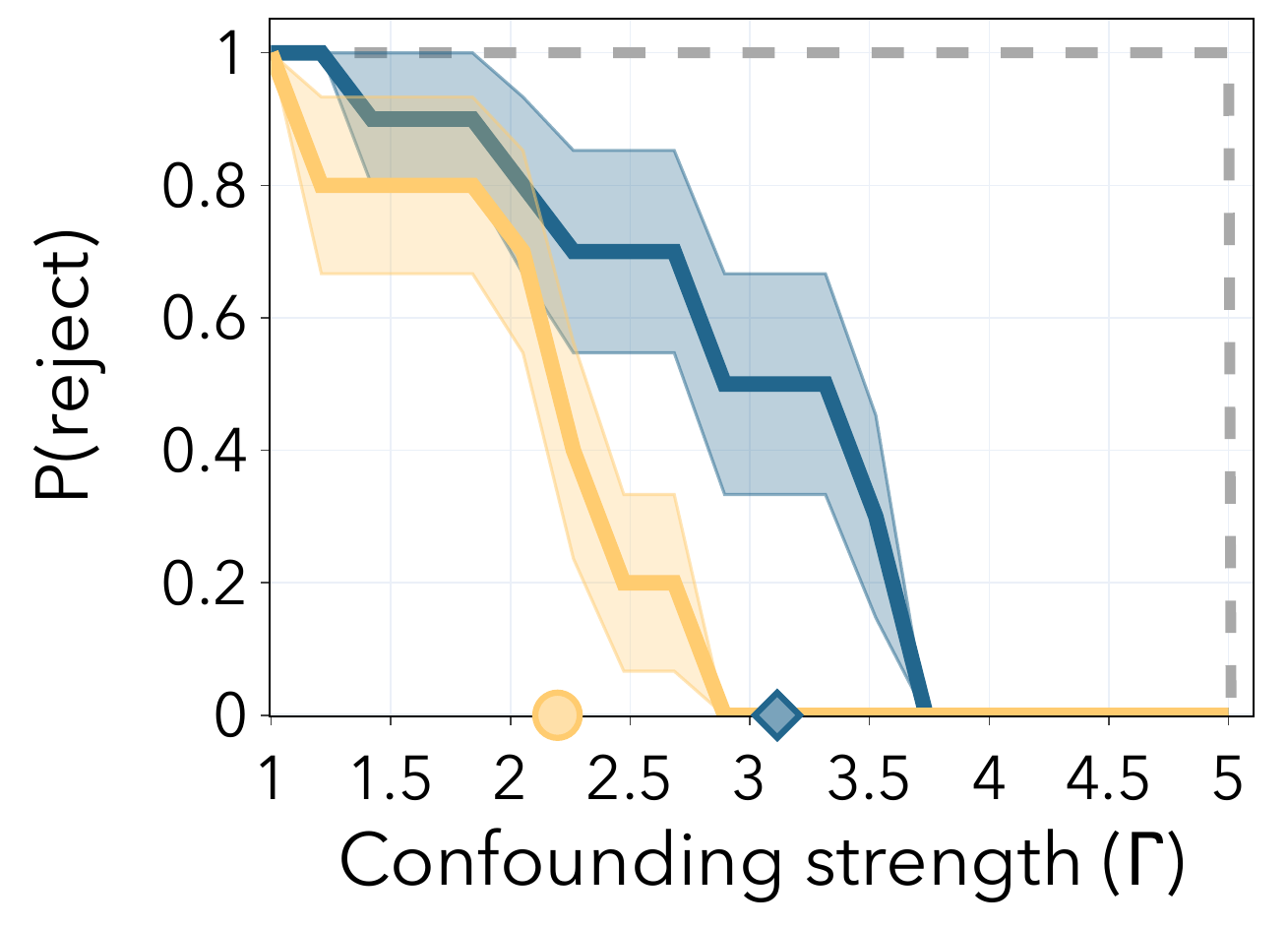}
        \caption{\footnotesize{Large sample (Hillstrom)} }
        \label{fig:hillstrom_large}
    \end{subfigure}
    \hfill
    \begin{subfigure}[b]{0.24\textwidth}
        \centering
   \includegraphics[width=\textwidth]{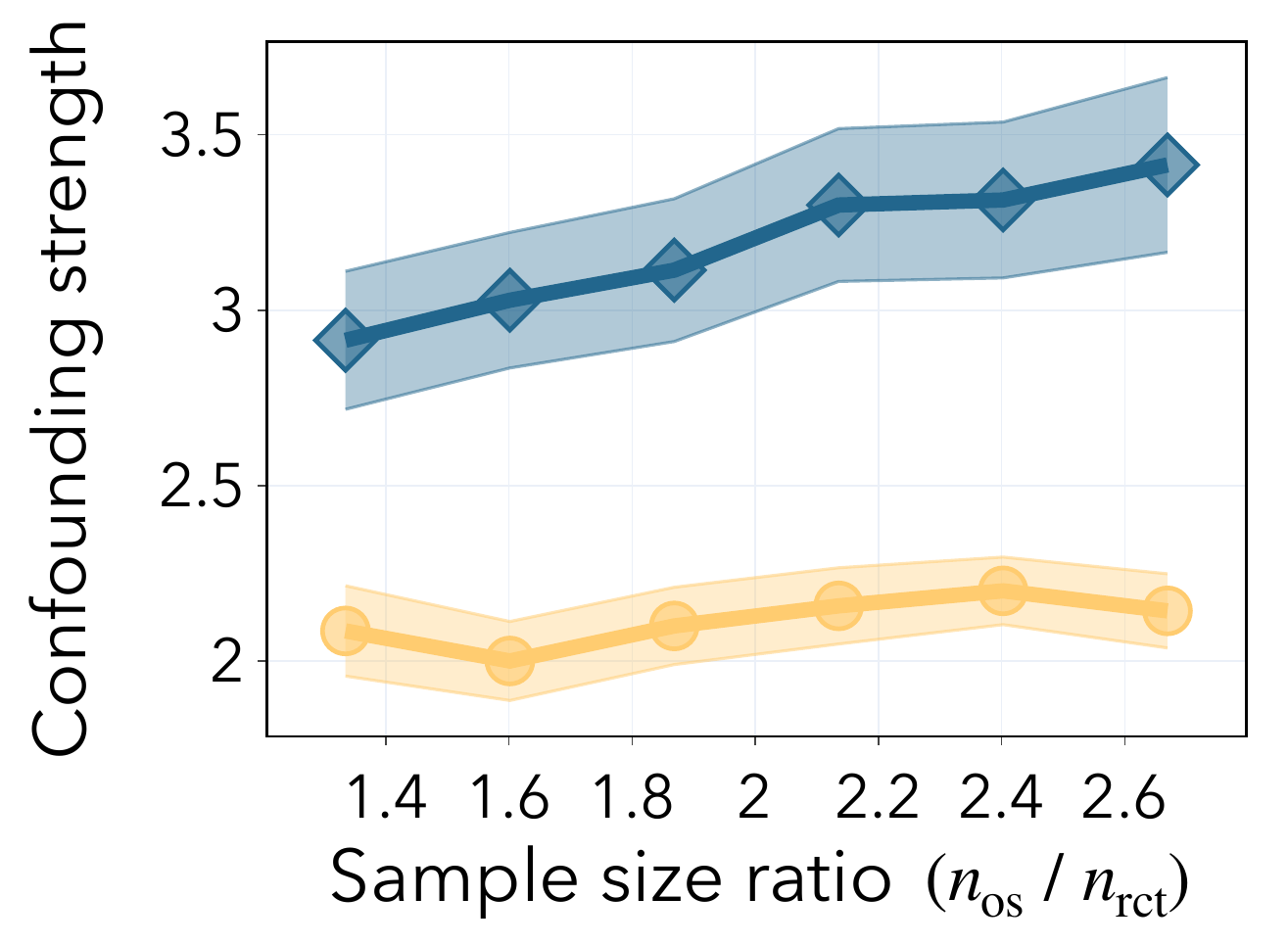}
        \caption{\footnotesize{Increasing $\nobs$} (Hillstrom)}
        \label{hillstrom:nobs}
    \end{subfigure}
    \hfill
    \begin{subfigure}[b]{0.245\textwidth}
        \centering
    \includegraphics[width=\textwidth]{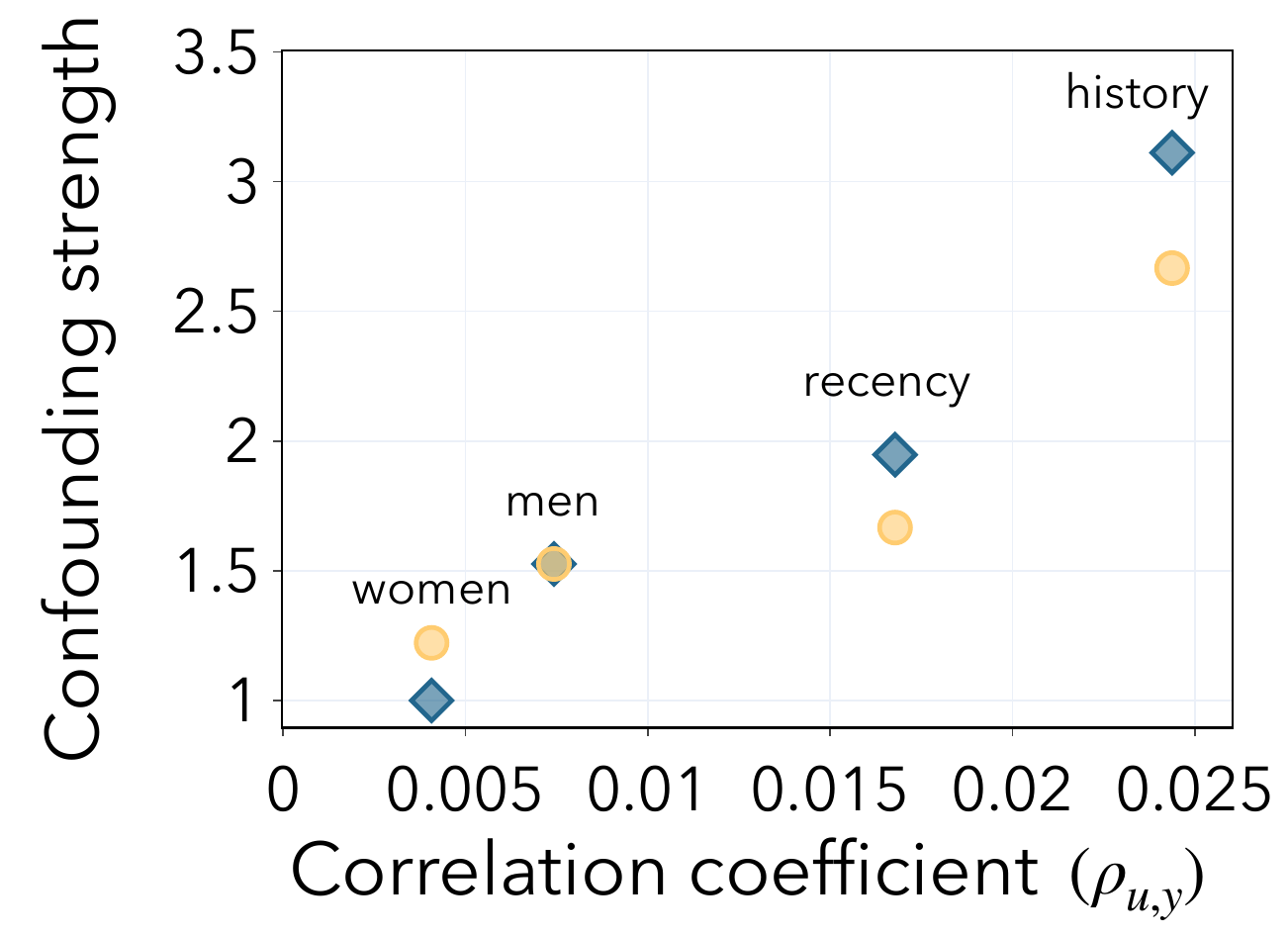}
        \caption{\footnotesize{$\rho_{u,y}$ effect (Hillstrom)}}
        \label{hillstrom:corr}

    \end{subfigure}
    
    \caption{\small{For all the plots: the significance level is $\alpha=0.05$, $\phi^\star$ denotes the oracle test which rejects for $\confvalue<
    \Gammatrue$, $\gammalb^{\rct}$ and $\gammalb^{\obsres}$ denote which test is used to compute $\gammalb$. First row with \emph{synthetic experiment} choosing  $\Gammatrue=5$: Probability of rejection for different $\confvalue$ and average $\gammalb$ for the test for (a) small sample size: $\nrct=2K,\nobs=2K$ and (b) large sample size: $\nrct=20K,\nobs=20K$. $\gammalb$ for (c) increasing sample size of the observational study with $\nrct=20K$ and (d) increasing correlation coefficient; $\nrct=20K,\nobs=20K$. Second row with the \emph{semi-synthetic} Hillstrom dataset choosing $\Gammatrue=5$ and using ``history" as unobserved confounder (except in (h)): Probability of rejection for different $\confvalue$ and average $\gammalb$ for (e) small sample size: $\nrct=2300$, $\nobs=6150$ and (f) large sample size: $\nrct=7680$, $\nobs=20500$. $\gammalb$ for (g) increasing $\nobs$ with $\nrct=7680$ and (h) increasing correlation coefficient.}
 }
    \label{fig:ate_synthetic}
\end{figure*}

\subsection{Experimental results}
 \label{sec:results}
 We now discuss our experimental results depicted in
~\Cref{fig:ate_synthetic}. The top row presents results for the synthetic experiments, and the bottom row for the semi-synthetic experiments.

\paragraph{Effect of observational study sample size} 

First, we observe in Figures \ref{fig:synthetic_small_sample}-\ref{fig:synthetic_large_sample} and Figures \ref{fig:hillstrom_small}-\ref{fig:hillstrom_large} that our tests are valid in all settings, i.e. they do not reject for strengths larger than $\Gammatrue$. 
However, the statistical power substantially improves in the large sample size regime.  
 In general, the performance of both tests aligns.
In Figures \ref{fig:synthetic_obs_sample} and \ref{hillstrom:nobs}, the lower bounds $\gammalb^{\rct}$ and $\gammalb^{\obsres}$ vary
with the sample size of the observational study. We confirm that the $\test_{\obsres}$  derives greater benefits from a larger observational study sample size than $\hat{\phi}_{\mathrm{rct}}$, as discussed in~\Cref{sec:test_presentetion}. 

\paragraph{Effect of outcome-confounder correlation} 
Note that the tests in ~\Cref{fig:synthetic_small_sample}-\ref{fig:synthetic_large_sample} and ~\Cref{fig:hillstrom_small}-\ref{fig:hillstrom_large} are somewhat conservative: The probability of rejection for $\confvalue$ close to $\Gammatrue$ is small, which leads to a rather loose lower bound estimate $\gammalb$. This is due to a fundamental limitation of the marginal sensitivity model that cannot be overcome without additional assumptions on how U affects Y, as discussed in~\Cref{apx:limitations_msm}. We study here the effect of increasing the outcome-confounder correlation  (Equation~\ref{eq:correlation}).
Specifically, we generate observational datasets with a constant $\Gammatrue$ but varying $\rho_{u,y}$, and report $\gammalb$ for both tests. For the synthetic experiments in~\Cref{fig:synthetic_correlation}, we plot $\gammalb^{\rct}$ and $\gammalb^{\obsres}$ for $n=50$ distinct values of $\sigma^2_{Y(1)} \sim \operatorname{Unif}[0,1]$. For the semi-synthetic experiments in~\Cref{hillstrom:corr}, we depict $\gammalb^{\rct}$ and $\gammalb^{\obsres}$ for different hidden confounders $U$. Both plots confirm our hypothesis that higher $\rho_{u,y}$ correlates with a tighter lower bound $\gammalb$.

\section{Real-world experiments}
\label{sec:whi_exp}

Linking back to the pipeline in \Cref{fig:motivating_example}, we demonstrate how epidemiologists can use the lower bound $\gammalb$ to successfully differentiate between studies with significant confounding and those with negligible confounding. 
Specifically, we propose comparing $\gammalb$ with a critical value of $\confvalue$, estimated from the available observational data
$$
\confvaluecrit \defeq \inf\{ \confvalue : 0 \in [\yonehat^-_{\confvalue}, \yonehat^+_{\confvalue}] \}. 
$$
In essence, $\confvaluecrit$ represents the minimum strength for which sensitivity analysis includes both positive and negative values of treatment effect, thereby invalidating the study results. Similar critical values have been proposed in the literature to assess the robustness of conclusions drawn from observational data, see e.g. \citet{vanderweele2017sensitivity,jin2023sensitivity}. The most appropriate choice for the specific context should be determined by epidemiologists. 

We flag an observational study as confounded if $\gammalb$
exceeds the critical value, i.e.
\begin{equation}
\label{test_critical}
\algours \defeq \indi \{ \gammalb > \confvaluecrit\}.
\end{equation}

We compare our decision-making procedure with one based on a binary test 
$$
\algbinary = \indi \{ \gammalb >1\}.
$$
In contrast to our procedure, the output of the binary one flags an observational study if any level of confounding is detected.  Note that choosing a more powerful binary test in the literature would only \emph{exacerbate} this issue. 

\paragraph{Controversy around HRT} For years, epidemiologists could not reach a consensus on the impact of hormone replacement therapy (HRT) on coronary heart disease and stroke
based on the findings of the Women's Health Initiative (WHI) study~\citep{anderson2003implementation}.
The WHI study included a randomized trial and an observational study that examined the impact of HRT on various cardiovascular events. While the observational study suggested that HRT had a protective effect against these outcomes, the randomized trial indicated the opposite. This discrepancy was recently resolved by identifying a strong unobserved confounder - the time $t$ since the start of HRT - and reanalyzing the data accordingly~\citep{vandenbroucke2009hrt}. We now present evidence that our procedure can yield the same epidemiological conclusions and avoid issuing false alarms when the confounding is negligible. 

\paragraph{Experimental details} We consider two binary-valued outcomes: the presence of stroke and coronary heart disease
within the follow-up period. We apply our procedure from \Cref{test_critical} to both the original dataset, which includes all patients (i.e. $t \leq 20$), and a subsampled dataset that only includes patients who were not previous users of HRT (i.e. $t = 0$). Since the WHI study satisfies the criteria for a nested trial design, we calculate $\gammalb$ using our testing procedure $\test_{\obsres}$.
See~\Cref{apx:whi_exp} for experimental details.

\begin{table}
\centering
\caption{The significance level is $\alpha=0.05$. For $t=0$ (small confounding), the study only included patients who were not previous users of HRT. For $t \leq 20$ (strong confounding), the study includes patients who have been using HRT for up to 20 years.}
\label{table:whi}
\begin{tabular}{cccccccc}
\toprule
\multirow{2}{*}{Metric} & \multicolumn{2}{c}{Stroke} & \multicolumn{2}{c}{Coronary heart disease} \\
\cmidrule(lr){2-3} \cmidrule(lr){4-5}
& $t=0$ & $t\leq20$ & $t=0$ & $t \leq 20$ \\
\midrule
$\confvaluecrit$ & 1.017 & 1.172 & 1.017 & 1.164 \\
$\gammalb$  & 1.052 & 1.207 & 1.009 & 1.224 \\
\midrule
$\algbinary$  & \textcolor{ForestGreen}{1} & \textcolor{ForestGreen}{1} & \textcolor{BrickRed}{1} & \textcolor{ForestGreen}{1} \\
$\algours$  & \textcolor{ForestGreen}{1} & \textcolor{ForestGreen}{1} & \textcolor{ForestGreen}{0} & \textcolor{ForestGreen}{1} \\
\bottomrule
\end{tabular}
\end{table}

\paragraph{Results} In~\Cref{table:whi}, we show the result of both procedures on the WHI dataset, with small ($t=0$) and large ($t \leq 20$) unobserved confounding .

For coronary heart disease, both algorithms flag the study as confounded when strong unobserved confounding is present ($t \leq 20$). However, when minimal unobserved confounding is present ($t = 0$), our test does not flag the study, while $\algbinary$ does. This difference underscores our test's capability to distinguish between small and large unobserved confounding, thereby addressing a limitation in the flagging procedures based on existing testing methods.

In the case of stroke, both $\algours$ and $\algbinary$ correctly flag the observational study, even when we adjust for the time since the start of treatment ($t=0$). This finding aligns with experts suggesting that additional unobserved confounding factors for stroke are still present after controlling for the time since the start of hormone replacement therapy~\citep{prentice2005combined}.

Observe that an alternative way to reach the same conclusions is by testing the difference in ATE estimates between the two studies. However, our approach offers a notable advantage: it allows us to test if the observational study is too confounded on arbitrarily fine-grained subgroups up to the individual level. Indeed, we can estimate the CATE sensitivity analysis bounds and compare critical values for specific subgroups against our lower bound. In contrast, testing differences in group-level ATE estimates would require several tests, one for each subgroup, leading to issues with multiple testing and insufficient sample sizes.

\section{Discussion and future work}
\label{sec:limitations}
Our approach shares limitations with other methods that test for unobserved confounding. Since we rely on the transportability assumption, our test could misidentify violations of this assumption as unobserved confounding. 
In addition, the lower bound we provide is optimistic; outside the common support of the two studies, the unobserved confounding could be arbitrarily high. Furthermore, our test is designed to detect confounding structures that bias the average treatment effect and, hence, would not detect confounding bias that cancels out on average.

Our discussion suggests several important directions for future research. 
First, 
developing a more refined sensitivity model that accounts for the correlation between outcomes and unobserved confounders could result in a more powerful test. Second, our test could be adapted to the scenario where multiple observational datasets may be available but no randomized control trials. Lastly, it would be highly valuable to propose a procedure that not only identifies hidden confounding but also suggests specific interventions to mitigate it. 

\subsection*{Acknowledgements}
PDB was supported by 
the Hasler Foundation grant number 21050. JA was supported by the ETH AI Center. KD was supported by the ETH AI Center and the ETH Foundations of Data Science.

\bibliography{main}

@article{dahabreh2021study,
  title={Study designs for extending causal inferences from a randomized trial to a target population},
  author={Dahabreh, Issa J and Haneuse, Sebastien JP A and Robins, James M and Robertson, Sarah E and Buchanan, Ashley L and Stuart, Elizabeth A and Hern{\'a}n, Miguel A},
  journal={American journal of epidemiology},
year={2021},
}

@article{veitch2020sense,
  title={Sense and sensitivity analysis: Simple post-hoc analysis of bias due to unobserved confounding},
  author={Veitch, Victor and Zaveri, Anisha},
  journal={Advances in Neural Information Processing Systems},
  year={2020}
}

@article{demirel2024benchmarking,
  title={Benchmarking Observational Studies with Experimental Data under Right-Censoring},
  author={Demirel, Ilker and De Brouwer, Edward and Hussain, Zeshan and Oberst, Michael and Philippakis, Anthony and Sontag, David},
  journal={International Conference on Artificial Intelligence and Statistics},
  year={2024}
}

@article{zhao2019sensitivity,
  title={Sensitivity analysis for inverse probability weighting estimators via the percentile bootstrap},
  author={Zhao, Qingyuan and Small, Dylan S and Bhattacharya, Bhaswar B},
  journal={Journal of the Royal Statistical Society Series B: Statistical Methodology},
  volume={81},
  number={4},
  pages={735--761},
  year={2019},
  publisher={Oxford University Press}
}

@article{degtiar2023review,
  title={A review of generalizability and transportability},
  author={Degtiar, Irina and Rose, Sherri},
  journal={Annual Review of Statistics and Its Application},
  volume={10},
  pages={501--524},
  year={2023},
  publisher={Annual Reviews}
}

@article{meinshausen2006quantile,
  title={Quantile regression forests.},
  author={Meinshausen, Nicolai},
  journal={Journal of machine learning research},
  volume={7},
  number={6},
  year={2006}
}

@article{hayden2016gene,
  title={Gene therapies pose million-dollar conundrum},
  author={Hayden, Erika Check},
  journal={Nature},
  volume={534},
  pages={305--306},
  year={2016}
}

@article{goetz2018personalized,
  title={Personalized medicine: motivation, challenges, and progress},
  author={Goetz, Laura H and Schork, Nicholas J},
  journal={Fertility and sterility},
  volume={109},
  number={6},
  pages={952--963},
  year={2018},
  publisher={Elsevier}
}

@article{cole2010generalizing,
  title={Generalizing evidence from randomized clinical trials to target populations: the {ACTG} 320 trial},
  author={Cole, Stephen and Stuart, Elizabeth},
  journal={American Journal of Epidemiology},
  volume={172},
  number={1},
  pages={107--115},
  year={2010},
  publisher={Oxford University Press}
}

@article{colnet2023risk,
  title={Risk ratio, odds ratio, risk difference... {W}hich causal measure is easier to generalize?},
  author={Colnet, B{\'e}n{\'e}dicte and Josse, Julie and Varoquaux, Ga{\"e}l and Scornet, Erwan},
  journal={arXiv preprint arXiv:2303.16008},
  year={2023}
}

@article{cornfield59,
    author = {Cornfield, Jerome and Haenszel, William and Hammond, E. Cuyler and Lilienfeld, Abraham M. and Shimkin, Michael B. and Wynder, Ernst L.},
    title = {Smoking and lung cancer: recent evidence and a discussion of some questions},
    journal = {JNCI: Journal of the National Cancer Institute},
    volume = {22},
    number = {1},
    pages = {173-203},
    year = {1959},
    month = {01},
}

@article{prentice2005combined,
  title={Combined postmenopausal hormone therapy and cardiovascular disease: toward resolving the discrepancy between observational studies and the \text{Women's Health Initiative} clinical trial},
  author={Prentice, Ross  and Langer, Robert and Stefanick, Marcia  and Howard, Barbara  and Pettinger, Mary and Anderson, Garnet and Barad, David and Curb,  David and Kotchen, Jane and Kuller, Lewis and others},
  journal={American Journal of Epidemiology},
  volume={162},
  number={5},
  pages={404--414},
  year={2005},
  publisher={Oxford University Press}
}

@article{lipsitch2010negative,
  title={Negative controls: a tool for detecting confounding and bias in observational studies},
  author={Lipsitch, Marc and Tchetgen, Eric Tchetgen and Cohen, Ted},
  journal={Epidemiology },
  volume={21},
  number={3},
  pages={383},
  year={2010},
  publisher={NIH Public Access}
}

@article{sofer2016negative,
  title={On negative outcome control of unobserved confounding as a generalization of difference-in-differences},
  author={Sofer, Tamar and Richardson, David  and Colicino, Elena and Schwartz, Joel and Tchetgen, Eric  Tchetgen},
  journal={Statistical Science: a Review Journal of the Institute of Mathematical Statistics},
  volume={31},
  number={3},
  pages={348},
  year={2016},
  publisher={NIH Public Access}
}

@article{de2014testing,
  title={Testing for the unconfoundedness assumption using an instrumental assumption},
  author={De Luna, Xavier and Johansson, Per},
  journal={Journal of Causal Inference},
  volume={2},
  number={2},
  pages={187--199},
  year={2014},
  publisher={De Gruyter}
}

@article{brantner2023methods,
  title={Methods for Integrating Trials and Non-Experimental Data to Examine Treatment Effect Heterogeneity},
  author={Brantner, Carly Lupton and Chang, Ting-Hsuan and Nguyen, Trang Quynh and Hong, Hwanhee and Di Stefano, Leon and Stuart, Elizabeth A},
  journal={arXiv preprint arXiv:2302.13428},
  year={2023}
}

@article{karlsson2023detecting,
  title={Detecting hidden confounding in observational data using multiple environments},
  author={Karlsson, Rickard and Krijthe, Jesse},
  journal={Advances in Neural Information Processing Systems},
  year={2023}
}

@article{richardson2013single,
  title={Single world intervention graphs (SWIGs): A unification of the counterfactual and graphical approaches to causality},
  author={Richardson, Thomas and Robins, James},
  journal={Center for the Statistics and the Social Sciences, University of Washington Series. Working Paper},
  volume={128},
  number={30},
  pages={2013},
  year={2013},
  publisher={Citeseer}
}

@article{colnet2020causal,
  title={Causal inference methods for combining randomized trials and observational studies: a review},
  author={Colnet, B{\'e}n{\'e}dicte and Mayer, Imke and Chen, Guanhua and Dieng, Awa and Li, Ruohong and Varoquaux, Ga{\"e}l and Vert, Jean-Philippe and Josse, Julie and Yang, Shu},
  journal={arXiv preprint arXiv:2011.08047},
  year={2020}
}

@article{donald2014testing,
  title={Testing the unconfoundedness assumption via inverse probability weighted estimators of {(L) ATT}},
  author={Donald, Stephen and Hsu, Yu-Chin and Lieli, Robert},
  journal={Journal of Business \& Economic Statistics},
  volume={32},
  number={3},
  pages={395--415},
  year={2014},
  publisher={Taylor \& Francis}
}

@article{vlahovic2011postmarketing,
  title={Postmarketing surveillance},
  author={Vlahovi{\'c}-Pal{\v{c}}evski, Vera and Mentzer, Dirk},
  journal={Pediatric Clinical Pharmacology},
  pages={339--351},
  year={2011},
  publisher={Springer}
}

@article{klonoff2020new,
  title={The new {FDA} real-world evidence program to support development of drugs and biologics},
  author={Klonoff, David},
  journal={Journal of Diabetes Science and Technology},
  volume={14},
  number={2},
  pages={345--349},
  year={2020},
  publisher={SAGE Publications Sage CA: Los Angeles, CA}
}

@article{hussain2023falsification,
  title={Falsification of Internal and External Validity in Observational Studies via Conditional Moment Restrictions},
  author={Hussain, Zeshan and Shih, Ming-Chieh and Oberst, Michael and Demirel, Ilker and Sontag, David},
  journal={International Conference on Artificial Intelligence and Statistics},
  year={2023},
}

@article{morucci2023double,
  title={A double machine learning approach to combining experimental and observational data},
  author={Morucci, Marco and Orlandi, Vittorio and Parikh, Harsh and Roy, Sudeepa and Rudin, Cynthia and Volfovsky, Alexander},
  journal={arXiv preprint arXiv:2307.01449},
  year={2023}
}

@article{viele2014use,
  title={Use of historical control data for assessing treatment effects in clinical trials},
  author={Viele, Kert and Berry, Scott and Neuenschwander, Beat and Amzal, Billy and Chen, Fang and Enas, Nathan and Hobbs, Brian and Ibrahim, Joseph and Kinnersley, Nelson and Lindborg, Stacy and others},
  journal={Pharmaceutical Statistics},
  volume={13},
  number={1},
  pages={41--54},
  year={2014},
  publisher={Wiley Online Library}
}

@article{kern2016assessing,
  title={Assessing methods for generalizing experimental impact estimates to target populations},
  author={Kern, Holger L and Stuart, Elizabeth A and Hill, Jennifer and Green, Donald P},
  journal={Journal of Research on Educational Effectiveness},
  volume={9},
  number={1},
  pages={103--127},
  year={2016},
  publisher={Taylor \& Francis}
}

@article{jin2023sensitivity,
  title={Sensitivity analysis of individual treatment effects: A robust conformal inference approach},
  author={Jin, Ying and Ren, Zhimei and Cand{\`e}s, Emmanuel J},
  journal={Proceedings of the National Academy of Sciences},
  volume={120},
  number={6},
  pages={e2214889120},
  year={2023},
  publisher={National Acad Sciences}
}

@article{vandenbroucke2009hrt,
  title={The {HRT} controversy: observational studies and {RCT}s fall in line},
  author={Vandenbroucke, Jan},
  journal={The Lancet},
  volume={373},
  number={9671},
  pages={1233--1235},
  year={2009},
  publisher={Elsevier}
}

@article{hussain2022falsification,
  title={Falsification before Extrapolation in Causal Effect Estimation},
  author={Hussain, Zeshan and Oberst, Michael and Shih, Ming-Chieh and Sontag, David},
  journal={Advances in Neural Information Processing Systems},
  year={2022}
}

@article{yadlowsky2022bounds,
  title={Bounds on the conditional and average treatment effect with unobserved confounding factors},
  author={Yadlowsky, Steve and Namkoong, Hongseok and Basu, Sanjay and Duchi, John and Tian, Lu},
  journal={The Annals of Statistics},
  volume={50},
  number={5},
  pages={2587--2615},
  year={2022},
  publisher={Institute of Mathematical Statistics}
}

@article{yangelastic,
    author = {Yang, Shu and Gao, Chenyin and Zeng, Donglin and Wang, Xiaofei},
    title = "{Elastic integrative analysis of randomised trial and real-world data for treatment heterogeneity estimation}",
    journal = {Journal of the Royal Statistical Society Series B: Statistical Methodology},
    volume = {85},
    number = {3},
    pages = {575-596},
    year = {2023},
    month = {04},    
       eprint = {https://academic.oup.com/jrsssb/article-pdf/85/3/575/50859810/qkad017.pdf},
}

@article{vanderweele2017sensitivity,
  title={Sensitivity analysis in observational research: introducing the E-value},
  author={VanderWeele, Tyler and Ding, Peng},
  journal={Annals of Internal Medicine},
  volume={167},
  number={4},
  pages={268--274},
  year={2017},
  publisher={American College of Physicians}
}

@article{tan2006distributional,
  title={A distributional approach for causal inference using propensity scores},
  author={Tan, Zhiqiang},
  journal={Journal of the American Statistical Association},
  volume={101},
  number={476},
  pages={1619--1637},
  year={2006},
  publisher={Taylor \& Francis}
}

@article{kallus2019interval,
  title={Interval estimation of individual-level causal effects under unobserved confounding},
  author={Kallus, Nathan and Mao, Xiaojie and Zhou, Angela},
  journal={International Conference on Artificial Intelligence and Statistics},
  year={2019},
}

@article{jesson2021quantifying,
  title={Quantifying ignorance in individual-level causal-effect estimates under hidden confounding},
  author={Jesson, Andrew and Mindermann, S{\"o}ren and Gal, Yarin and Shalit, Uri},
  journal={International Conference on Machine Learning},
  year={2021},
}

@article{oprescu2023b,
  title={B-learner: Quasi-oracle bounds on heterogeneous causal effects under hidden confounding},
  author={Oprescu, Miruna and Dorn, Jacob and Ghoummaid, Marah and Jesson, Andrew and Kallus, Nathan and Shalit, Uri},
 journal={International Conference on Machine Learning},
  year={2023},
}

@article{dorn2021doubly,
  title={Doubly-valid/doubly-sharp sensitivity analysis for causal inference with unmeasured confounding},
  author={Dorn, Jacob and Guo, Kevin and Kallus, Nathan},
  journal={arXiv preprint arXiv:2112.11449},
  year={2021}
}

@article{dorn2022sharp,
  title={Sharp sensitivity analysis for inverse propensity weighting via quantile balancing},
  author={Dorn, Jacob and Guo, Kevin},
  journal={Journal of the American Statistical Association},
  pages={1--13},
  year={2022},
  publisher={Taylor \& Francis}
}

@article{degtiar2021conditional,
author = {Degtiar, Irina and Layton, Tim and Wallace, Jacob and Rose, Sherri},
title = {Conditional cross-design synthesis estimators for generalizability in {M}edicaid},
journal = {Biometrics},
volume = {00},
pages = {1 -- 14},
year={2023}
}

@article{anderson2003implementation,
  title={Implementation of the {W}omen's {H}ealth {I}nitiative study design},
  author={Anderson, Garnet and Manson, Joann and Wallace, Robert and Lund, Bernedine and Hall, Dallas and Davis, Scott and Shumaker, Sally and Wang, Ching-Yun and Stein, Evan and Prentice, Ross},
  journal={Annals of Epidemiology},
  volume={13},
  number={9},
  pages={S5--S17},
  year={2003},
  publisher={Elsevier}
}

@article{hsu2013calibrating,
  title={Calibrating sensitivity analyses to observed covariates in observational studies},
  author={Hsu, Jesse Y and Small, Dylan S},
  journal={Biometrics},
  volume={69},
  number={4},
  pages={803--811},
  year={2013},
  publisher={Wiley Online Library}
}

@article{platt2018fda,
  title={The {FDA Sentinel Initiative}—an evolving national resource},
  author={Platt, Richard and Brown, Jeffrey and Robb, Melissa and McClellan, Mark and Ball, Robert and Nguyen, Michael and Sherman, Rachel},
  journal={New England Journal of Medicine},
  volume={379},
  number={22},
  pages={2091--2093},
  year={2018}
}

@article{dreyer2018advancing,
  title={Advancing a framework for regulatory use of real-world evidence: when real is reliable},
  author={Dreyer, Nancy A},
  journal={Therapeutic Innovation and Regulatory Science},
  volume={52},
  number={3},
  pages={362--368},
  year={2018},
  publisher={SAGE Publications Sage CA: Los Angeles, CA}
}

@article{keith2023rct,
  title={{RCT} Rejection Sampling for Causal Estimation Evaluation},
  author={Keith, Katherine A and Feldman, Sergey and Jurgens, David and Bragg, Jonathan and Bhattacharya, Rohit},
  journal={arXiv preprint arXiv:2307.15176},
  year={2023}
}

@article{gerber2008social,
  title={Social pressure and voter turnout: evidence from a large-scale field experiment},
  author={Gerber, Alan S and Green, Donald P and Larimer, Christopher W},
  journal={American political Science review},
  volume={102},
  number={1},
  pages={33--48},
  year={2008},
  publisher={Cambridge University Press}
}

@misc{hillstrom2008,
  author = {Hillstrom, Kevin},
  year = {2008},
  title = {The {MineThatData} e-mail analytics and data mining challenge},
  publisher = {MineThatData blog}
}

@article{olschewski1985comprehensive,
  title={Comprehensive cohort study: an alternative to randomized consent design in a breast preservation trial},
  author={Olschewski, Manfred and Scheurlen, H},
  journal={Methods of Information in Medicine},
  volume={24},
  number={03},
  pages={131--134},
  year={1985},
  publisher={Schattauer GmbH}
}

@article{olschewski1992analysis,
  title={Analysis of randomized and nonrandomized patients in clinical trials using the comprehensive cohort follow-up study design},
  author={Olschewski, Manfred and Schumacher, Martin and Davis, Kathryn B},
  journal={Controlled Clinical Trials},
  volume={13},
  number={3},
  pages={226--239},
  year={1992},
  publisher={Elsevier}
}

@article{choudhry2017randomized,
  title={Randomized, controlled trials in health insurance systems},
  author={Choudhry, Niteesh K},
  journal={New England Journal of Medicine},
  volume={377},
  number={10},
  pages={957--964},
  year={2017},
  publisher={Mass Medical Soc}
}

@article{word1990state,
  title={The State of {T}ennessee’s student/teacher achievement ratio ({STAR}) Project},
  author={Word, Elizabeth and Johnston, John and Bain, Helen P and Fulton, B DeWayne and Zaharias, Jayne B and Achilles, Charles M and Lintz, Martha N and Folger, John and Breda, Carolyn},
  journal={Tennessee Board of Education},
  year={1990}
}

@inproceedings{gentzel2021and,
  title={How and why to use experimental data to evaluate methods for observational causal inference},
  author={Gentzel, Amanda M and Pruthi, Purva and Jensen, David},
  booktitle={International Conference on Machine Learning},
  pages={3660--3671},
  year={2021},
  organization={PMLR}
}

@article{franklin2019evaluating,
  title={Evaluating the use of nonrandomized real-world data analyses for regulatory decision making},
  author={Franklin, Jessica and Glynn, Robert and Martin, David and Schneeweiss, Sebastian},
  journal={Clinical Pharmacology \& Therapeutics},
  volume={105},
  number={4},
  pages={867--877},
  year={2019},
  publisher={Wiley Online Library}
}

@article{schurman2019framework,
  title={The framework for {FDA}’s real-world evidence program},
  author={Schurman, Beth},
  journal={Applied Clinical Trials},
  volume={28},
  number={4},
  year={2019},
  publisher={MJH Life Sciences}
}

@article{kallus2018removing,
  title={Removing hidden confounding by experimental grounding},
  author={Kallus, Nathan and Puli, Aahlad Manas and Shalit, Uri},
  journal={Advances in Neural Information Processing Systems},
  year={2018}
}

@article{stuart2011use,
  title={The use of propensity scores to assess the generalizability of results from randomized trials},
  author={Stuart, Elizabeth and Cole, Stephen and Bradshaw, Catherine and Leaf, Philip},
  journal={Journal of the Royal Statistical Society Series A: Statistics in Society},
  volume={174},
  number={2},
  pages={369--386},
  year={2011},
  publisher={Oxford University Press}
}

@article{hartman2015sample,
  title={From sample average treatment effect to population average treatment effect on the treated: combining experimental with observational studies to estimate population treatment effects},
  author={Hartman, Erin and Grieve, Richard and Ramsahai, Roland and Sekhon, Jasjeet S},
  journal={Journal of the Royal Statistical Society Series A: Statistics in Society},
  volume={178},
  number={3},
  pages={757--778},
  year={2015},
  publisher={Oxford University Press}
}

@article{andrews2017weighting,
  title={Weighting for external validity},
  author={Andrews, Isaiah and Oster, Emily},
  year={2017},
  journal={NBER: National Bureau of Economic Research}
}

@article{colnet2022reweighting,
  title={Reweighting the {RCT} for generalization: finite sample analysis and variable selection},
  author={Colnet, B{\'e}n{\'e}dicte and Josse, Julie and Varoquaux, Ga{\"e}l and Scornet, Erwan},
  journal={arXiv preprint arXiv:2208.07614},
  year={2022}
}

@article{nie2021covariate,
  title={Covariate balancing sensitivity analysis for extrapolating randomized trials across locations},
  author={Nie, Xinkun and Imbens, Guido and Wager, Stefan},
  journal={arXiv preprint arXiv:2112.04723},
  year={2021}
}

@mastersthesis{stutz2023,
  author  = "Stutz, Alexander",
  title   = "Can Semi-Supervised Learning Improve the Estimation of Causal Treatment Effects?",
  school  = "ETH Zurich",
  year    = "2023",
  month   = "August"
}

@article{chen2015xgboost,
  title={Xgboost: extreme gradient boosting},
  author={Chen, Tianqi and He, Tong and Benesty, Michael and Khotilovich, Vadim and Tang, Yuan and Cho, Hyunsu and Chen, Kailong and Mitchell, Rory and Cano, Ignacio and Zhou, Tianyi and others},
  journal={R package version 0.4-2},
  volume={1},
  number={4},
  pages={1--4},
  year={2015}
}

@article{sugden1984ignorable,
  title={Ignorable and informative designs in survey sampling inference},
  author={Sugden, RA and Smith, TMF},
  journal={Biometrika},
  volume={71},
  number={3},
  pages={495--506},
  year={1984},
  publisher={Oxford University Press}
}

@article{hotz2005predicting,
  title={Predicting the efficacy of future training programs using past experiences at other locations},
  author={Hotz, V Joseph and Imbens, Guido W and Mortimer, Julie H},
  journal={Journal of econometrics},
  volume={125},
  number={1-2},
  pages={241--270},
  year={2005},
  publisher={Elsevier}
}

@article{o2014generalizing,
  title={Generalizing from unrepresentative experiments: a stratified propensity score approach},
  author={O'Muircheartaigh, Colm and Hedges, Larry V},
  journal={Journal of the Royal Statistical Society Series C: Applied Statistics},
  volume={63},
  number={2},
  pages={195--210},
  year={2014},
  publisher={Oxford University Press}
}

@article{he2020clinical,
  title={Clinical trial generalizability assessment in the big data era: a review},
  author={He, Zhe and Tang, Xiang and Yang, Xi and Guo, Yi and George, Thomas and Charness, Neil and Quan Hem, Kelsa Bartley and Hogan, William and Bian, Jiang},
  journal={Clinical and Translational Science},
  volume={13},
  number={4},
  pages={675--684},
  year={2020},
  publisher={Wiley Online Library}
}
\bibliographystyle{plainnat}

\clearpage

\section*{Appendices}
The following appendices provide deferred proofs, experiment details, and ablation studies.
\appendix
\DoToC
\clearpage
\hypersetup{
    colorlinks,
   linkcolor={pierLink},
    citecolor={pierCite},
    urlcolor={pierCite}
}
\section{Methodology}
\subsection{Remaining proofs}
We present here the proofs 
for~\Cref{lemma:identify_ate,prop:test}.
\subsubsection{Proof of~\Cref{lemma:identify_ate}}
\label{apx:proof_id_ate}
For $\diamond=\rct$, we have
\begin{align*}
	\mu(\pxrct, \pfullrct) &= \EE_{\pfullrct} [ Y(1)-Y(0)] \\
	&= \EE_{\pxrct} \left[ \EE_{\prct}[Y \mid T=1,X]\frac{\prct(T=1)}{\prct(T=1)}  -\EE_{\prct}[Y \mid T=0,X]\frac{\prct(T=0)}{\prct(T=0)}  \right]\\
	&= \EE_{\prct} \left[ Y\left(\frac{T}{\pi} - \frac{(1-T)}{1-\pi}  \right)  \right],
\end{align*}
where the last equality follows from the internal validity of the randomized trial.

For $\diamond=\obsres$, we first note that by transportability of CATE and definition of $\pobsres$, we have
\begin{equation*}
\mu(\pobsres_X, \pfullobsres) = \mu(\pobsres_X, \pfullobs) 
= \EE_{\pobsres_X}\left[ \EE_{\pfullobs}[Y(1) - Y(0) \mid X] \right]  = \EE_{\pobsres_X}\left[ \EE_{\pfullrct}[Y(1) - Y(0) \mid X] \right]. 
\end{equation*}

Furthermore, it holds via \Cref{asm:nested_support} (support inclusion) and the definition of $\pobsres$ that
\begin{align*}
\EE_{\prct_X}\left[ \EE_{\pfullrct}[Y(1) - Y(0) \mid X] \frac{\pobsres(X)}{\prct(X)} \right] 
&=\EE_{\pxrct}\left[ \left( \EE_{\prct}[Y \mid T=1, X] - \EE_{\prct}[Y \mid T=0, X]  \right) \frac{\pobsres(X)}{\prct(X)} \right] \\
&= \EE_{\prct}\left[ Y\left( \frac{T}{\pi} - \frac{(1-T)}{1-\pi} \right) \frac{\pobsres(X)}{\prct(X)} \right], 
\end{align*}
where the last equality again
follows from the internal validity of the randomized trial.

\subsubsection{Proof of~\Cref{prop:test}}
\label{apx:proof_test}
First, observe that by definition,
\begin{equation}
\{\test_\diamond(\confvalue,\alpha) = 1\} \implies
\{\hat{T}_\confvalue^+  \leq z_{\alpha/2} \} \cup \{  \hat{T}_\confvalue^-  \leq z_{\alpha/2} \}. 
\end{equation} 
Hence, if $\p_{\hnull(\Gamma)}(\hat{T}_\confvalue^-  \leq z_{\alpha/2}) \leq \frac{\alpha}{2} + o(1)$ and $\p_{\hnull(\Gamma)}(\hat{T}_\confvalue^+  \leq z_{\alpha/2}) \leq \frac{\alpha}{2} + o(1)$, the result follows from the union bound.
For brevity, we only prove $\hat{T}_\confvalue^+ \leq z_{\alpha/2}$ as the proof for $\hat{T}_\confvalue^-$ is analogous.

\paragraph{Proof of case $\diamond=\rct$} 
 Let $(X_i,Y_i,T_i)$ be i.i.d. sampled from $\prct$ and define \begin{align*}Z&=\left(\frac{YT}{\pi} - \frac{Y(1-T)}{1-\pi}, \quad\cate_\confvalue^+(X)\right)^T,\quad  \text{with} \\ \mu &\defeq \EE_{\prct}[Z] = \left(\yone\left(\pxrct, \pfullrct\right), \quad \yone_\confvalue^+ \left(\pxrct\right)\right)^T \quad \text{and} \quad \Sigma:=\operatorname{Cov}_{\prct}(Z)= \left(\begin{array}{cc}\sigma^2 & \xi_\confvalue  \\ \xi_\confvalue & (\sigma^+_\confvalue)^2 \end{array} \right) < \infty.\end{align*}
Further, we define  
$$
\bar{Z}_n=\left( \frac{1}{\nrct}  \sum_{(T_i,Y_i)  \in \datarct}\frac{Y_iT_i}{\pi} - \frac{Y_i(1-T_i)}{1-\pi}   , \frac{1}{\nrct}  \sum_{X_i  \in \datarct} \cate_\confvalue^+(X_i) \right).
$$
By the multivariate central limit theorem, we have 
$$
\sqrt{\nrct}\left(\bar{Z}_n-\mu\right) \stackrel{\mathcal{D}}{\longrightarrow} \mathcal{N}(0, \Sigma), \;\;\text{as}\;\;\nrct \to \infty.
$$
Thus, it follows from the Cramér-Wold theorem that 
$$
\sqrt{\nrct}\left[\left(\frac{1}{\nrct}  \sum_{X_i  \in \datarct} \cate_\confvalue^+(X_i)  - \yonehat\right)-\left(\yone_\confvalue^+\left(\pxrct\right) - \yone\left(\pxrct, \pfullrct \right)\right)\right]\xrightarrow{\mathcal{D}} \gauss \left(0,  \sigma^2  +(\sigma^+_\confvalue)^2-2\xi_\confvalue\right),\;\;\text{as}\;\;\nrct \to \infty.$$
It remains to show that asymptotic normality also holds when we replace the oracle average $\frac{1}{\nrct}\sum_{X_i \in \datarct} \cate_\confvalue^+(X_i)$ by the empirical estimator $\yonehat_\confvalue^+$. To do so, define
\begin{equation}
\label{eq:convg}
R_n \defeq \sqrt{\nrct}\left(\frac{1}{\nrct}\sum_{X_i \in \datarct} \cate_\confvalue^+(X_i) - \yonehat_\confvalue^+(X_i)\right).
\end{equation}
Let $d_n(x) \defeq \cate_\confvalue^+(x) - \catehat_\confvalue^+(x)$. Then, conditional on the observational dataset $\dataobs$ used to train $\yonehat^+_\confvalue(\cdot)$, 
\begin{align*}
\left|\EE_{\prct}\left[R_n \mid \dataobs\right]\right|
\leq \sqrt{\nrct}\,\EE_{\prct}\left[|d_n(X)| \mid \dataobs\right]
\leq \sqrt{\nrct}\,\|d_n\|_{L^2(\prct)}
= O_{\pobs}\left(\sqrt{\nrct/\nobs}\right)
= o_{\pobs}(1),
\end{align*}
where we used $\|d_n\|_{L^2(\prct)} = O_{\pobs}(\nobs^{-1/2})$ and $\nrct/\nobs \to 0$. Similarly,
\[
\var_{\prct}\left[R_n \mid \dataobs\right]
\leq \EE_{\prct}\left[d_n(X)^2 \mid \dataobs\right]
= \|d_n\|_{L^2(\prct)}^2
= O_{\pobs}(\nobs^{-1})
= o_{\pobs}(1).
\]
Hence, by Chebyshev's inequality, for every $\varepsilon > 0$,
\[
\prct\left(|R_n|>\varepsilon \mid \dataobs\right)
\leq \frac{1}{\varepsilon^2}\left(\var_{\prct}\left[R_n \mid \dataobs\right] + \EE_{\prct}\left[R_n \mid \dataobs\right]^2\right)
= o_{\pobs}(1).
\]
Since $0 \leq \prct\left(|R_n|>\varepsilon \mid \dataobs\right) \leq 1$, this implies
\[
\mathbb P\left(|R_n|>\varepsilon\right)
= \EE_{\pobs}\left[\prct\left(|R_n|>\varepsilon \mid \dataobs\right)\right]
\longrightarrow 0,
\]
where $\mathbb P$ denotes the joint law of $(\dataobs,\datarct)$. Thus, $R_n=o_{\mathbb P}(1)$ under the joint law, and by the asymptotic normality established above and Slutsky's theorem, we have, as $\nrct \to \infty$ and $\nobs \to \infty$,
\begin{equation}
	\label{eq:as_norm_1}
\sqrt{\nrct}\left[\left(\yonehat_\confvalue^+ - \yonehat\right)-\left(\yone_\confvalue^+\left(\pxrct\right) - \yone\left(\pxrct, \pfullrct \right)\right)\right]\xrightarrow{\mathcal{D}} \gauss \left(0,  \sigma^2  +(\sigma^+_\confvalue)^2 -2\xi_\confvalue\right).
\end{equation}

Finally, by the consistency of $\sigmahat^2, \varubhat$ and Slutsky's theorem, we have
\begin{align*}
\lim_{\nrct,\nobs \to \infty} \mathbb P_{\hnull(\Gamma)} \left(\hat{T}_\confvalue^+  \leq z_{\alpha/2} \right)&= \lim_{\nrct,\nobs \to \infty} \mathbb P_{\hnull(\Gamma)}\left(\frac{\yonehat^{+}_\confvalue - \yonehat }{\sqrt{\varubhat + \sigmahat^2 +2\sigmahat^+_\confvalue\sigmahat} }  \leq z_{\alpha/2} \right)\\
&=\lim_{\nrct,\nobs \to \infty} \mathbb P_{\hnull(\Gamma)}\left(\frac{\sqrt{\nrct}(\yonehat^{+}_\confvalue - \yonehat)}{\sqrt{(\sigma^+_\confvalue)^2 + \sigma^2 +2\sigma^+_\confvalue \sigma} }  \leq z_{\alpha/2} \right) \\ 
&\leq \lim_{\nrct,\nobs \to \infty} \mathbb P_{\hnull(\Gamma)}\left(\frac{\sqrt{\nrct}(\yonehat^{+}_\confvalue - \yonehat)}{\sqrt{(\sigma^+_\confvalue)^2 + \sigma^2-2\xi_\confvalue} }  \leq z_{\alpha/2} \right),
\end{align*}
where in the last line, we use the Cauchy--Schwarz covariance inequality, i.e. $|\xi_\confvalue| \leq \sigma^+_\confvalue\sigma$. Further, under $\hnull(\Gamma)$, it holds that $\yone_\confvalue^+\left(\pxrct\right) -\yone\left(\pxrct, \pfullrct \right)\geq 0$ by~\Cref{lemma:test_1}. Then, by asymptotic normality established in~\Cref{eq:as_norm_1}, we conclude that
\begin{align*}
\lim_{\nrct,\nobs \to \infty} \mathbb P_{\hnull(\Gamma)}\left(\hat{T}_\confvalue^+  \leq z_{\alpha/2} \right)&\leq \lim_{\nrct,\nobs  \to \infty} \mathbb P_{\hnull(\Gamma)} \left(\frac{\sqrt{\nrct}\left((\yonehat^{+}_\confvalue - \yonehat) -\yone_\confvalue^+\left(\pxrct\right) +\yone\left(\pxrct, \pfullrct \right)\right)}{\sqrt{(\sigma_\confvalue^+)^2 + \sigma^2-2\xi_\confvalue} }  \leq z_{\alpha/2} \right) \\
&= \Phi(z_{\alpha/2})=\alpha/2. 
\end{align*}

\paragraph{Proof of case $\diamond=\obsres$}
Let $n=\nrct+\nobs$ with fixed proportions, where $\nrct=\rho n$ and $\nobs=(1-\rho)n$ for $\rho \in (0,1)$. By the central limit theorem and~\Cref{lemma:identify_ate}, it holds that
$$
\sqrt{n}\left(\frac{1}{\nrct}\sum_{\substack{(X_i,T_i,Y_i)  \in \datarct}} \left(\frac{Y_iT_i}{\pi} - \frac{Y_i(1-T_i)}{1-\pi}  \right)w(X_i) - \yone\left(\pobsres_X, \pfullobs \right)\right)\xrightarrow{\mathcal{D}} \gauss\left(0, \sigma^2/\rho \right) \quad \text{as} \quad n \to \infty.$$
Then, from the asymptotic normality of $\yonehat_\confvalue^+$ and the independence $\yonehat_\confvalue^+\indep \yonehat$, we have
$$
\sqrt{n}\left(\begin{array}{c} \yonehat_\confvalue^+ - \yone_\confvalue^+  \\\yonehat - \yone \end{array}\right)
\stackrel{\mathcal{D}}{\rightarrow} \mathcal{N}\left(\left(\begin{array}{l}0 \\ 0 \end{array}\right),\left[\begin{array}{cc}(\sigma_\confvalue^+)^2/(1-\rho) & 0  \\ 0 & \sigma^2/\rho \end{array} \right]\right).
$$
Hence, by the $\delta$-technique with $h(X) = X_1 - X_2$, we get 
$$
\sqrt{n}\left(\yonehat_\confvalue^+ - \yonehat - \yone_\confvalue^+\left(\pobsres_X \right) + \yone\left(\pobsres_X, \pfullobs \right)\right)\xrightarrow{\mathcal{D}} \gauss \left(0, \frac{(\sigma^+_\confvalue)^2}{1-\rho} + \frac{\sigma^2}{\rho} \right) \quad \text{as} \quad  n \to \infty.
$$
Finally, from the consistency of $\sigmahat^2, \varubhat$ and Slutsky's theorem,  it holds that
$$
\frac{\yonehat^{+}_\confvalue - \yonehat - \yone_\confvalue^+\left(\pobsres_X \right) + \yone\left(\pobsres_X, \pfullobs \right)}{\sqrt{\varubhat+ \sigmahat^2}} \xrightarrow{\mathcal{D}}  \gauss\left(0, 1\right) \quad \text{as} \quad  n \to \infty.
$$
As before,  asymptotic normality implies that
\begin{align*}
\lim_{n\to \infty} \mathbb P_{\hnull(\Gamma)} \left(\hat{T}_\confvalue^+  \leq z_{\alpha/2} \right)&= \lim_{n\to \infty} \mathbb P_{\hnull(\Gamma)} \left(\frac{\yonehat^{+}_\confvalue - \yonehat }{\sqrt{\varubhat + \sigmahat^2} }  \leq z_{\alpha/2} \right)\\
&\leq \lim_{n\to \infty} \mathbb P_{\hnull(\Gamma)}\left(\frac{\yonehat^{+}_\confvalue - \yonehat -\yone_\confvalue^+\left(\pobsres_X \right) + \yone\left(\pobsres_X, \pfullobs \right) }{\sqrt{\varubhat + \sigmahat^2} }  \leq z_{\alpha/2} \right)
\\&= \Phi(z_{\alpha/2})=\alpha/2.
\end{align*}


\subsection{Nested design}
\label{apx:nested_design}
 In a nested trial design, the randomized trial is embedded in a cohort of eligible people who are proposed to participate in the trial, but if they refuse, they are still included in the observational study. 
Two concrete examples of nested designs are the Women Health Initiative~\citep{anderson2003implementation} and the recent study on Medicaid~\citep{degtiar2021conditional}.

In contrast to~\Cref{sec:setting}, the nested design has an extra variable $S \in \{0,1\}$, which is a binary indicator for randomized trial participation.
 More formally, we observe i.i.d. samples from an underlying distribution $\pfullnest$ over $\left(X, U, Y(0), Y(1),Y, S, T\right)$. Further, let $\pnest \defeq \marg(\pfullnest)$ be the marginal distribution  over $(X,Y,S,T)$. We can then  write the marginal distributions over $X$ of the  (restricted) observational study and randomized trial  as 
$$
\prct_X(X) =  \pnest\left(X \mid S=1 \right) \quad \mathrm{and}  \quad 
\pobsres_X(X) = \pnest\left(X \mid S=0, X \in \supp(\prct) \right).$$
This study design has a significant impact on the estimators previously introduced. In particular, the importance weights $w(X)$ can be estimated by pooling the observational study and the randomized trial as follows 
\begin{align*}
	w(X) &= \frac{\pobsres(X)}{\prct(X)} \\&= \frac{\pnest(X\mid S=0, X\in \supp(\prct))}{\pnest(X \mid S=1)}\\
	&= \frac{\pnest(S=0\mid X) }{\pnest(S=1\mid X)} \frac{\pnest(S=1 \mid X \in \supp(\prct))}{\pnest(S=0 \mid X \in \supp(\prct))},
\end{align*}
where the sampling probability $\pnest(S\mid X)$ can be identified under a nested study design~\citep{dahabreh2021study}.

\subsection{Limitations of MSM}
\label{apx:limitations_msm}
We discuss here the intuition behind the tightness of $\gammalb$ in the infinite-sample limit, though it carries over to finite samples. Without loss of generality, we focus on the lower bound derived from $\test_{\obsres}$ and assume that the unobserved confounding biases the average treatment effect upwards, i.e.
$$
\yone^-_{\confvalue=1}(\pobsres_X) = \yone^+_{\confvalue=1}(\pobsres_X) \geq \yone(\pobsres_X, \pfullobsres),
$$
where $\yone^-_{\confvalue=1}(\pobsres_X)$ and $\yone^+_{\confvalue=1}(\pobsres_X)$ are the IPW estimates of ATE on the (restricted) marginal distribution of the observational study.
Further, we define the infinite-sample lower bound as
$$
\gammalbinfinite \defeq \inf \left \{ \confvalue: \lim_{n\to \infty} \test_{\obsres}(\confvalue,\alpha) =0 \right \},
$$
which is, in words, the smallest value of $\confvalue$ such that the sensitivity bounds include the true ATE.
The key observation is that the interval $[\yone^-_\confvalue(\pobsres_X),\yone^+_\confvalue(\pobsres_X)]$ widens with $\confvalue$: the lower sensitivity bound is non-increasing and the upper sensitivity bound is non-decreasing in $\confvalue$.
Under the upward-bias assumption above, the interval first contains the true ATE when the lower bound drops below it. Hence,
\begin{align*}
\gammalbinfinite
= \inf \left \{ \confvalue: \yone^-_{\confvalue}(\pobsres_X) \leq \yone(\pobsres_X,\pfullobsres) \right \}.
\end{align*}
 If one additionally assumes that $\confvalue \mapsto \yone^-_{\confvalue}(\pobsres_X)$ is continuous and strictly decreasing in a neighborhood of the crossing, then the inequality above can be replaced by equality.

The looseness of the lower bound can be characterized by
$\Delta \defeq \Gammatrue - \gammalbinfinite$: if $\Delta > 0$ even in the infinite-sample limit, the lower bound will not be tight. We discuss two interesting cases:

\paragraph{Case 1.} If
		$$
		\pfullobsres \in \underset{\pfulltilde \in \msm(\pobsres, \Gammatrue)}{\argmin}\, \yone(\pobsres_X,\pfulltilde),
		$$
	then $\yone_{\Gammatrue}^-(\pobsres_X) = \yone(\pobsres_X,\pfullobsres)$ and therefore $\Delta=0$. This corresponds to the case in which the observed full distribution is itself extremal for the lower ATE bound on the target population. It is also in line with worst-case constructions where the hidden confounder encodes the potential outcomes, for example $U=(Y(1),Y(0))$; see~\citet{dorn2022sharp} for a closed-form solution of the ATE sensitivity bounds.

\paragraph{Case 2.} If $U \indep (Y(1),Y(0))$, then $\yone^-_{\confvalue=1}(\pobsres_X) = \yone^+_{\confvalue=1}(\pobsres_X) = \yone(\pobsres_X, \pfullobsres)$. Hence, $\Delta=\Gammatrue -1$ and the lower bound can be arbitrarily loose.

\label{../apx:limitations_msm}
\section{Experimental details}
\subsection{Synthetic experiments}
\label{apx:syn_exp}
We design the propensity scores such that the data distribution satisfies the MSM with true confounding strength $\confvalue^*=5$. To do so, we define the \emph{adversarial propensity score} as 
\begin{align*}
&e^+(X,U) = \begin{cases}
        \ell(X) & \text{if } U>t(X) \\
        u(X) & \text{if }  U \leq t(X)
    \end{cases}, \quad \text{where} \quad \ell(X)=\frac{e(X)}{e(X)+(1-e(X))\Gammatrue},  &u(X)=\frac{e(X)}{e(X)+(1-e(X)) / \Gammatrue}
\end{align*}
are respectively the lower and upper bounds on the full propensity score under the MSM. 
By choosing $t(X) = \frac{e(X) - \ell(X)}{u(X) -\ell(X)}$ in our data-generating process in~\Cref{sec:datasets} where $U \sim U(0,1)$, we ensure 
that $\EE_{\pobs}[e^+(X,U) \mid X] = e(X)$. 
We note that this is different from~\cite{kallus2019interval,jesson2021quantifying,oprescu2023b} where they choose a fixed threshold $t(X)=1/2$, resulting in a data distribution that does not satisfy the MSM. For all synthetic experiments, we estimate the propensity score using logistic regression. Further, we set $\nbs = 100, \sigma^2_Y=0.1, \alpha=0.05$, and report the mean and standard error over 20 runs.

For the test $\test_{\obsres}$, we use the sensitivity bound estimator QB~\citep{dorn2022sharp}, and fit the quantile function using quantile forest regression~\citep{meinshausen2006quantile}. For the test $\test_{\rct}$, we use the sensitivity bound estimator B-Learner~\citep{oprescu2023b}. We fit the quantile function using quantile forest regression~\citep{meinshausen2006quantile}, and the outcome model using a random forest regressor.

\subsection{Semi-synthetic experiments}
\label{apx:semisyn_exp}
We provide the details of the semi-synthetic experiments in \Cref{sec:datasets}. Specifically, we describe the subsampling procedure used to generate a randomized trial and an observational study that satisfy our setting, along with additional information about the datasets employed.

\subsubsection{Subsampling procedure}

\begin{algorithm}
\caption{Randomized Trial Rejection Sampling \citep{keith2023rct}}
\label{algo:rejection}
\begin{algorithmic}[1]
\State \textbf{Inputs:} $D \setminus \datarct = \{(X_i, U_i, Y_i,T_i)\}_{i=1}^{n}$; $\pconfobs(T = 1 \mid U)$, a function specified by the user; $M$ 
a constant computed empirically.
\State \textbf{Output:} $\dataobs$.
\State $\dataobs \leftarrow D \setminus \datarct$
\While{true}
\For{each unit $i$ in $\dataobs$}
\State Sample $K_i$ uniform on $(0, 1)$
\If{$K_i > \frac{\pconfobs(T=T_i \mid U_i)}{\widehat{\p}(T=t_i) M}$}
\State  $\dataobs \leftarrow \dataobs \setminus \{(X_i, U_i, Y_i,T_i)\}$
\EndIf
\EndFor
\State \textbf{break} if no units were discarded in the last iteration
\EndWhile
\State Remove $U$ from $\dataobs$
\end{algorithmic}
\end{algorithm}

We now detail the procedure for constructing a randomized trial and multiple observational datasets for our semi-synthetic experiments. 
Given a large-scale real-world randomized trial $D$ with covariates $\Xall$, our objective is to create multiple observational datasets $\dataobs$ that differ in the correlation $\rho_{u,y} = \frac{\operatorname{Cov}_{\pfullobs}[Y(1),U]}{\sigma_{Y(1)} \sigma_U}$, i.e. $\pinv$, but have the same confounding strength $\Gammatrue$, i.e. $\pconfobs$. This setup allows us to separately understand the effect of $\rho_{u,y}$ on the power of the test. While we cannot directly intervene on $\pinv (Y(1), Y(0)|\Xall)$ as it is intrinsic to the dataset, we can hide different $U \in \Xall$ for each $\dataobs$, resulting in different  $\pinv (Y(1), Y(0)|U)$  and hence correlation coefficient $\rho_{u,y}$. 

For each candidate hidden confounder $U$ within $\Xall$, we implement the following steps:  First, we select a subset from $D$ to construct our randomized trial dataset, $\datarct$, and remove $U$ from $\Xall$.  Next, we \emph{subsample} $D \setminus \datarct$ to generate a dataset $\dataobs$ that belongs to $\msm(\pobs, \Gammatrue)$. We enforce this constraint by constructing a specific propensity score $\pconfobs(T = 1\mid U)$  (detailed in the sequel) and employing the subsampling procedure in~\Cref{algo:rejection} using this  $\pconfobs(T = 1\mid U)$. Note that for simplicity, we choose a propensity score that does not depend on $X$; this is consistent with the graphical model in Figure~\ref{fig:dag}.

We use different $\pconfobs(T = 1\mid U)$ for continuous and binary confounders.
We first define
\begin{equation}
\label{eq:uplow}
\ell = \frac{\hat{\pi}}{\hat{\pi} + [1-\hat{\pi}]\Gammatrue}, \quad
u = \frac{\hat{\pi}}{\hat{\pi} + [1-\hat{\pi}]/\Gammatrue},
\end{equation}
where we estimate $\hat{\pi} = \widehat{\p}(T = 1)$ from the split $D \setminus \datarct$ before subsampling.
For a \emph{continuous} confounder $U$ positively correlated with $Y(1)$, we use the following full propensity score in~\Cref{algo:rejection}:
\begin{equation}
\label{adv_semi}
\pconfobs(T=1 \mid U) = 
\begin{cases}
    \ell & \text{if } U > Q_{\hat{q}^*}(U), \\
    u & \text{if } U \leq Q_{\hat{q}^*}(U),
\end{cases}
\end{equation}
where $Q_q(U) = \inf \left\{ z \in \mathbb{R} : {\pconfobs(U \leq z)} \geq q \right\}$ is the $q$-th quantile of the marginal distribution of $U$ in $D \setminus \datarct$ before subsampling.
Since the propensity score does not depend on $X$, using $\hat{q}^\star = \frac{\hat{\pi} - \ell}{u - \ell}$ makes sure that 
$$e(X)=\pconfobs(T = 1|X)=\pconfobs(T = 1) = \hat \pi$$
and hence the target full propensity score satisfies the marginal constraint. 
Note that this is equivalent to enforcing the same marginal propensity score before and after the subsampling. 
For a negatively correlated continuous confounder, we choose $\hat{q}^\star = \frac{u - \hat{\pi}}{u - \ell}$ and change the direction of the inequalities in~\Cref{adv_semi}.

Further, for a positively correlated \emph{binary} confounder, we use the following propensity score in~\Cref{algo:rejection}:
\begin{equation}
\label{adv_semi_binary}
\pconfobs(T=1 \mid U) = 
\begin{cases}
    \ell & \text{if } U = 1, \\
    u & \text{if } U = 0.
\end{cases}
\end{equation}
By first subsampling $D \setminus \datarct$ such that
\begin{align*}
\widehat{\p}(U = 1) &= \frac{u - \hat{\pi}}{u - \ell},
\end{align*}
and then applying~\Cref{algo:rejection} with the propensity score from~\Cref{adv_semi_binary}, we again obtain $e(X)=\pconfobs(T = 1) = \hat \pi$.
For a negatively correlated binary confounder, we first enforce
$\widehat{\p}(U = 1) = \frac{\hat{\pi} - \ell}{u - \ell}$ and swap $\ell$ for $u$ in~\Cref{adv_semi_binary}, and vice versa.

The subsampling procedure in \Cref{algo:rejection} allows for the construction of observational datasets where the causal effect is identifiable non-parametrically, in contrast to previous approaches, such as \citep{gentzel2021and}, that do not guarantee identifiability after subsampling. In \Cref{algo:rejection}, 
we note that  $M=\frac{\max _{i \in\{1, \ldots, n\}} \pconfobs\left(T=t_i \mid U_i\right)}{\min _{i \in\{1, \ldots, n\}} \widehat{\p}\left(T=t_i\right)}$, which is required to ensure that the causal effect is identifiable in the subsampled dataset (see Theorem 3.2 in~\cite{keith2023rct}). 
We empirically observe that the subsampling procedure discards half of the instances from  $D \setminus \datarct$.

\subsubsection{Datasets details}
\label{apx:dataset_details}

We give additional details about the three datasets we use for the semi-synthetic experiments.

\begin{itemize}
    \item \textbf{Hillstrom's MineThatData Email data}~\citep{hillstrom2008}. The Hillstrom dataset contains records of 64,000 customers who purchased within the last twelve months. They were part of an e-mail campaign to assess the effectiveness of distinct campaign strategies. Two treatment groups, ``Men’s" and ``Women’s" email campaigns, and a control group were established. Treatments were randomly assigned. Our analysis primarily focuses on a combined treatment group, which constitutes roughly 66\% of the dataset. While the original dataset has different outcomes, we looked at the dollars spent in the two weeks post-campaign. The dataset provides data on recent purchase patterns (Recency), annual spending categories (History Segment) and values (History), merchandise type, either Mens (Mens) or Womens (Womens), geographical location via zip code (Zip Code), newcomer status (Newbie), and purchasing avenues (Channel). After subsampling, we end up with a randomized trial of size $\nrct=7680$ and an observational dataset of size $\nobs=20500$. \Cref{asm:nested_support} is enforced by excluding urban zip codes in the trial.

    \item \textbf{VOTE dataset}~\citep{gerber2008social}. The VOTE dataset studies the effect of social pressure on voting behaviors among Michigan's registered voters, focusing on those who voted in the prior election and met certain criteria. The primary outcome is a binary variable indicating whether the letter recipients voted. In this randomized trial, participants were allocated to a control group or one of four treatment groups. The treatment groups received distinct letters, each varying in social pressure intensity, aimed at encouraging voting. The most persuasive letter provided insight into the recipient's neighbors' voting patterns from the previous two elections and implied updates on neighbors' subsequent voting actions in future letters. Using the split in \citep{stutz2023}, we incorporated roughly 190,000 samples in the control group, and we kept the treatment group with the strongest letter, leaving about 38,000 samples. We retained preprocessed features like age, household size, gender, and two scores reflecting past voting habits and the voting patterns of neighbors. After subsampling, we end up with a randomized trial of size $\nrct=10650$ and an observational dataset of size $\nobs=36800$.  We discard households with more than 4 participants to enforce \Cref{asm:nested_support}.

    \item \textbf{Tennessee STAR Project}~\citep{word1990state}. The Tennessee STAR experiment, initiated in 1985, was a randomized study examining the impact of class size on students' standardized test scores, tracking them from kindergarten through third grade. Initially, students and teachers were randomly placed into class sizes, intending to maintain these conditions throughout the study. We follow the dataset preprocessing outlined in~\citep{kallus2018removing}.  Their analysis concentrates on two conditions: small classes (13-17 students) and regular-sized classes (22-25 students). They used the class size at first grade as the treatment variable, observing 4,509  students. Their outcome aggregates scores from listening, reading, and math tests at the end of the first grade. After excluding students with missing values, the final sample consisted of 4,218 students: 1,805 in small classes (treatment) and 2,413 in regular-sized classes (control). The observed features for each student are gender, race, birth month, birthday, birth year, free lunch given or not and teacher ID.  After subsampling, we end up with a randomized trial of size $\nrct=600$ and an observational dataset of size $\nobs=1800$. For the STAR project, we keep inner-city and suburban students but remove urban and rural ones to enforce~\Cref{asm:nested_support}.
\end{itemize}

We one-hot-encode all categorical features and standardize in $[0, 1]$ all continuous features.

\paragraph{Implementation }We use QB~\citep{dorn2022sharp} for the test $\test_{\obsres}$. For continuous outcomes (Hillstrom and STAR), we fit a random forest regressor for the quantile functions, while we leverage the closed-form solution for the quantiles in the binary case (VOTE). For the test $\test_{\rct}$ we use the B-Learner~\citep{oprescu2023b}, fitting respective random forest regressors for the quantiles, the outcome and the bounds models. For the binary case, we use the closed-form solution for the quantiles and fit the outcome and bounds models with XGBoost~\citep{chen2015xgboost}. We always train a logistic regressor for the propensity score.  We report mean and standard error over 15 runs and set $\nbs = 200, \alpha=0.05$ for all experiments.

\subsection{Women's Health Initiative}
\label{apx:whi_exp}
The Women's Health Initiative (WHI) is a long-term national health study that has focused on strategies for preventing the major causes of death, disability, and frailty in older women, specifically heart disease, cancer, and osteoporotic fractures. This multi-million dollar, 20+ year project, sponsored by the National Institutes of Health (NIH), the National Heart, Lung, and Blood Institute (NHLBI), originally enrolled 161,808 women aged 50-79 between 1993 and 1998. The WHI was one of the most definitive, far-reaching clinical trials of post-menopausal women's health ever undertaken in the US.  

The WHI had two major parts: a Clinical Trial and an Observational Study. The randomized controlled Clinical Trial (CT) enrolled 68,132 women on trials testing three prevention strategies. Eligible women could choose to enrol in one, two, or three of the trial components.
\begin{itemize}
\item Hormone Therapy Trials (HT): This component examined the effects of combined hormones or estrogen alone on the prevention of heart disease and osteoporotic fractures, and associated risk for breast cancer. Women participating in this component took hormone pills or a placebo (inactive pill) until the Estrogen plus Progestin and Estrogen Alone trials were stopped early in July 2002 and March 2004, respectively. All HT participants continued to be followed without intervention until close-out.
\item 	Dietary Modification Trial (DM): The Dietary Modification component evaluated the effect of a low-fat and high-fruit, vegetable and grain diet on the prevention of breast and colorectal cancers and heart disease. Study participants followed either their usual eating pattern or a low-fat dietary pattern.
\item Calcium/Vitamin D Trial (CaD): This component evaluated the effect of calcium and vitamin D supplementation on the prevention of osteoporotic fractures and colorectal cancer. Women in this component took calcium and vitamin D pills or placebos.
\end{itemize}
The Observational Study (OS) examines the relationship between lifestyle, health and risk factors and disease outcomes. This component involves tracking the medical events and health habits of 93,676 women. Recruitment for the observational study was completed in 1998 and participants have been followed since.

To assess our method in a real-world scenario, we use observational study and randomized trial data from the Women’s Health Initiative (WHI). We use the Postmenopausal Hormone Therapy (PHT) trial as the RCT in our analysis ($\nrct=16,608$), which was run on postmenopausal women aged 50-79 years with an intact uterus. The trial investigated the effect of hormone therapy on several types of cancers, cardiovascular events, and fractures, measuring the “time-to-event” for each outcome. In the WHI setup, the observational study component was run in parallel and tracked similar outcomes to the RCT.

\paragraph{Data preprocessing} 
We binarize a composite outcome, called the “global index”, in our analysis, where Y = 1 if coronary heart disease or stroke was observed in the first seven years of follow-up, and Y = 0 otherwise. Note that Y = 0 could also occur from censoring. To establish treatment and control groups in the observational study, we use questionnaire data in which participants confirm or deny usage of combination hormones (i.e. both estrogen and progesterone) in the first three years. Using this procedure, we end up
 with a total of $\nobs = 33,511$ patients. Finally, we restrict the set of covariates used to those that are measured in both the RCT and the observational study.  In particular, we use as covariates only those measured in both the RCT and observational study, and we further restrict them to those identified as significant in epidemiological literature, such as in~\citep{prentice2005combined}. Specifically, the covariates in our analysis are: \texttt{AGE}, \texttt{ETHNIC\_White}, \texttt{BMI}, \texttt{SMOKING\_Past\_Smoker}, \texttt{SMOKING\_Current\_Smoker}, \texttt{EDUC\_x\_College\_graduate\_or\_Baccalaureate Degree}, \texttt{EDUC\_x\_Some\_post-graduate\_or\_professional}, \texttt{MENO}, \texttt{PHYSFUN}. The data used is available on \href{(https://biolincc.nhlbi.nih.gov/studies/whi_ctos)}{BIOLINCC}.

\paragraph{Experimental details}We train logistic regression for both outcome models and propensity score. We use as sensitivity bounds DVDS~\citep{dorn2021doubly} for the test $\test_{\obsres}$, and B-Learner for the test $\test_{\rct}$. We test for confounding in one direction, i.e. we only compute the test statistic $\hat{T}_\confvalue^+$. We set $\nbs=100$ and $\alpha=0.05$ for all experiments.

\label{apx:experimental_details}
\clearpage
\section{Additional Experiments}

\subsection{VOTE dataset}

We present the experimental results with the VOTE dataset in~\Cref{fig:vote_results}. Experiments were conducted with both weak and strong confounders, and under small and large sample regimes. We use the outcome $Y$ as a strong confounder given the lack of a feature highly correlated with the outcome in the dataset. These results corroborate previous findings that higher correlated confounders and larger sample sizes lead to greater power of our test. In all scenarios, the performance of both tests closely aligns.

\begin{figure*}[hb]
    \begin{subfigure}[b]{0.24\textwidth}
        \centering
    \includegraphics[width=\textwidth]{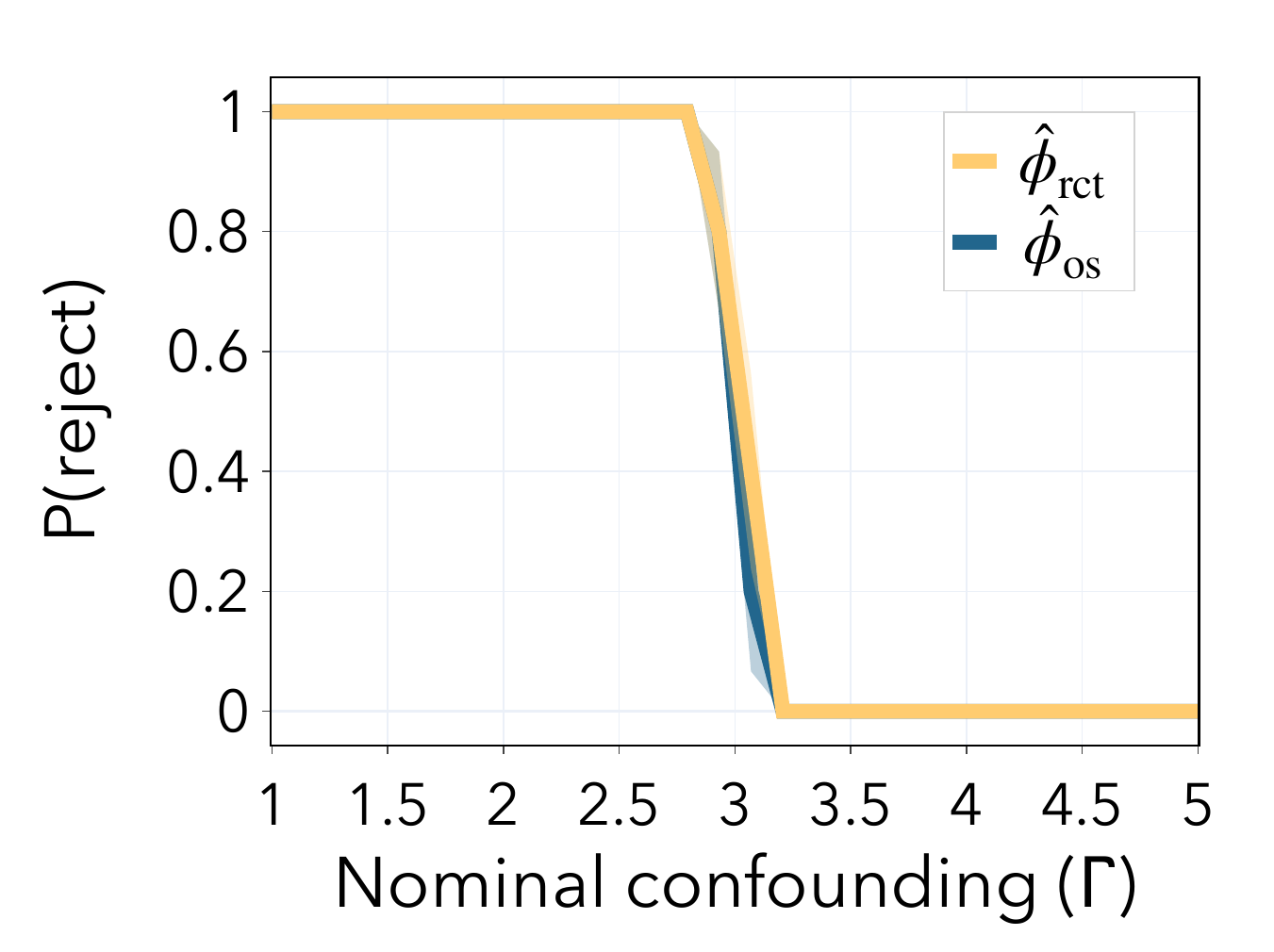}
        \caption{Small sample ``age"}
        \label{vote:weak_small}
    \end{subfigure}
    \hfill
    \begin{subfigure}[b]{0.24\textwidth}
        \centering
\includegraphics[width=\textwidth]{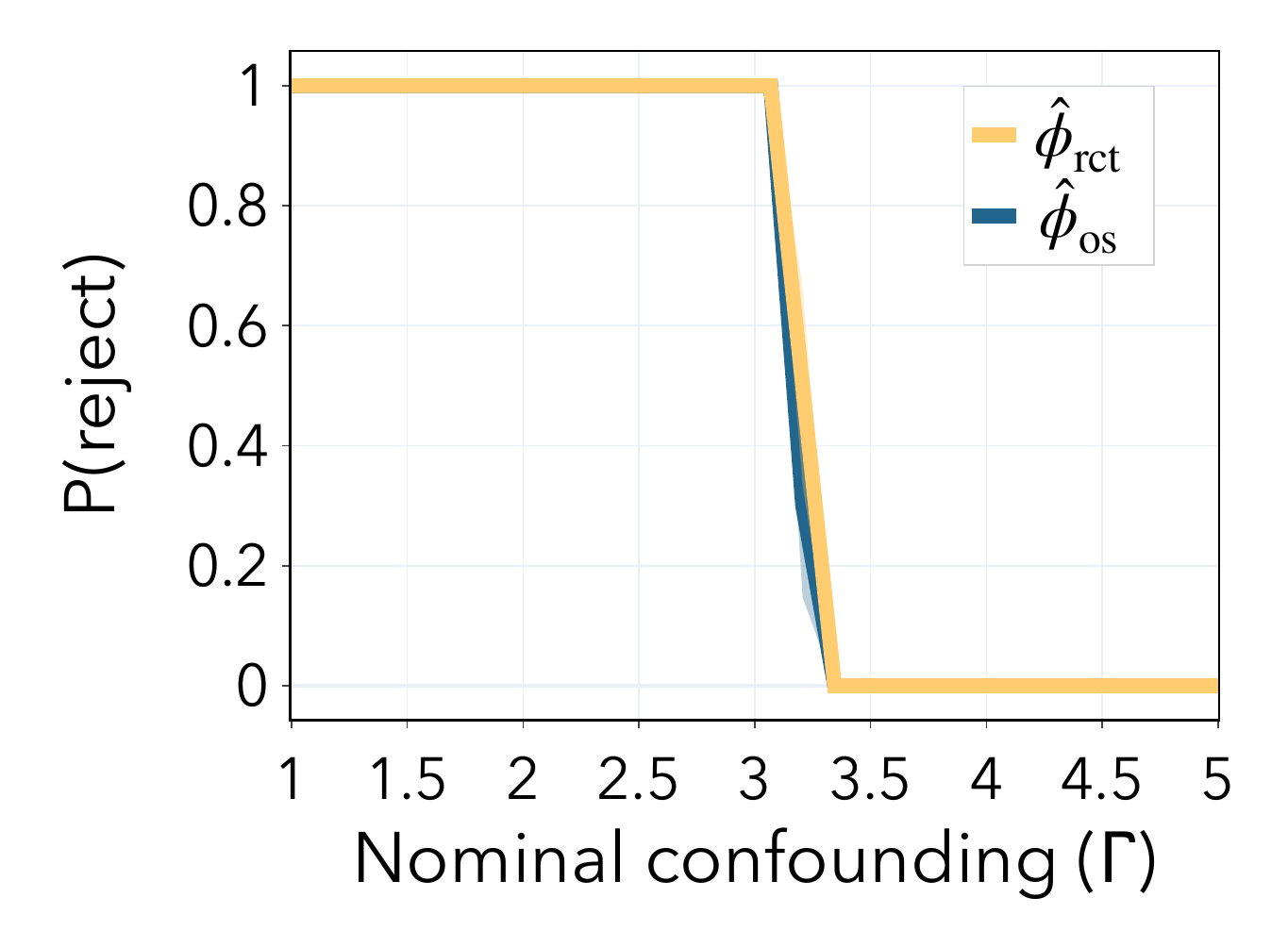}
        \caption{Large sample ``age"}
        \label{vote:weak_large}
    \end{subfigure}
    \hfill
    \begin{subfigure}[b]{0.24\textwidth}
        \centering
   \includegraphics[width=\textwidth]{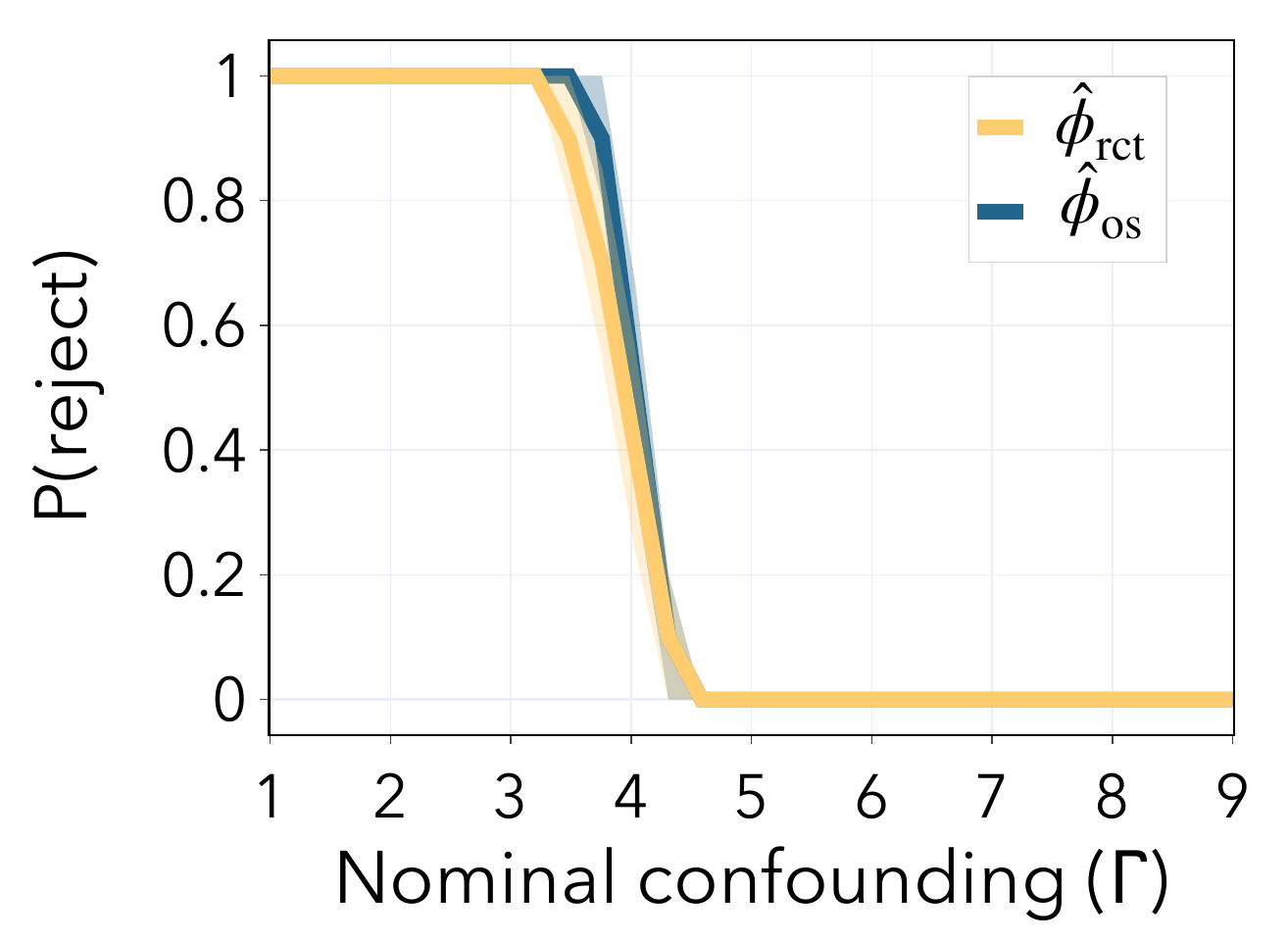}
        \caption{Small sample $Y$}
        \label{vote:strong_small}
    \end{subfigure}
    \hfill
    \begin{subfigure}[b]{0.245\textwidth}
        \centering
    \includegraphics[width=\textwidth]{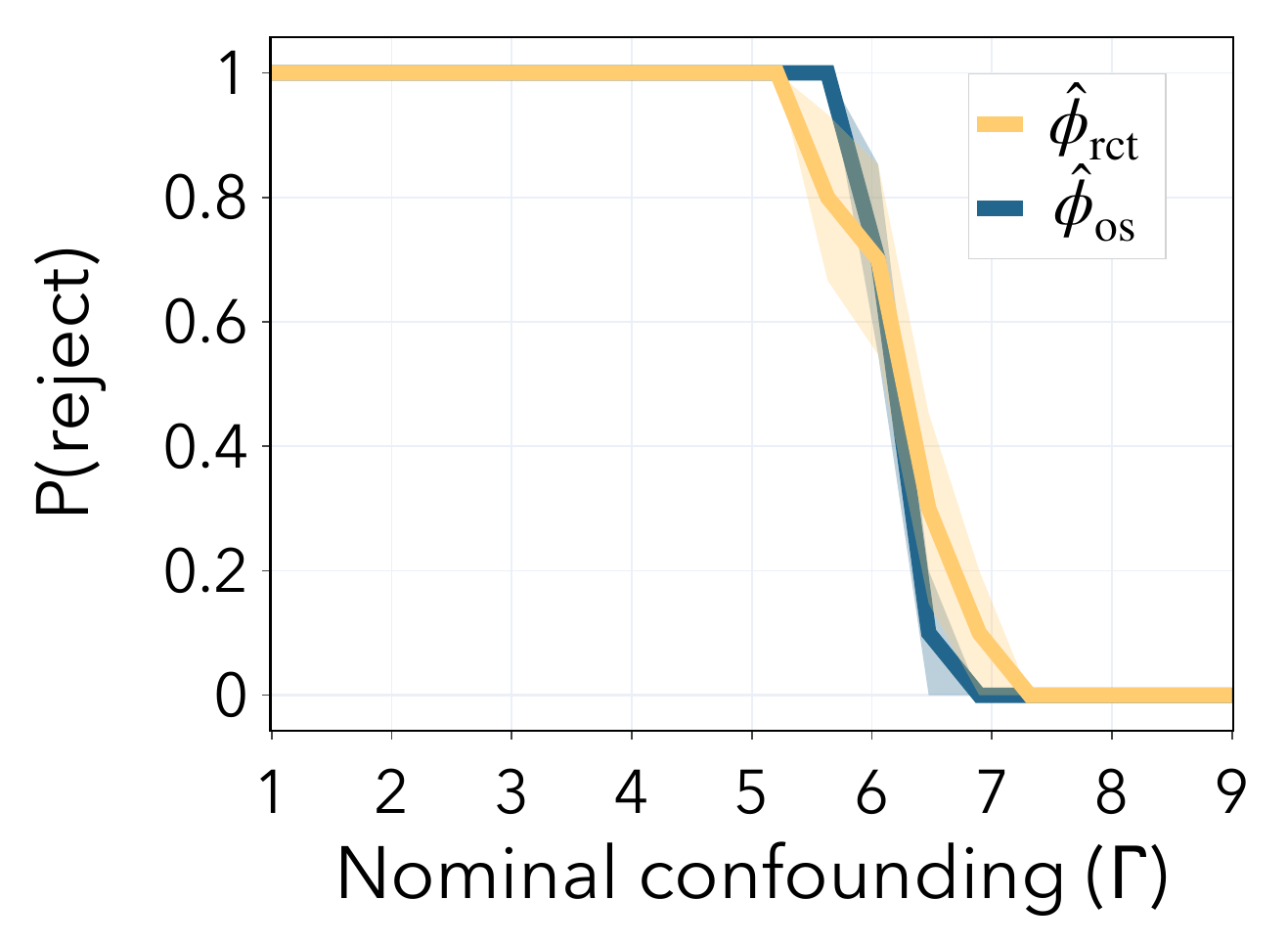}
        \caption{Large sample $Y$}
        \label{vote:strong_large}

    \end{subfigure}
    
    \caption{Probability of rejection for different choices of $\confvalue$ for the test for the VOTE dataset. For all the plots, the significance level is $\alpha=0.05$ and $\Gammatrue=9$. (a)-(b) Weak confounder: ``age". (a) small sample size: $\nrct=3.2K,\nobs=11K$ and (b) large sample size: $\nrct=10.6K,\nobs=36.8K$. (c)-(d) Strong confounder: outcome $Y$. (c) small sample size: $\nrct=3.2K,\nobs=11K$ and (d) large sample size: $\nrct=10.6K,\nobs=36.8K$.
 }
    \label{fig:vote_results}
\end{figure*}

\subsection{Tennessee STAR Project}

 We present the experimental results with the STAR Project in~\Cref{fig:star_results}. Experiments were conducted with both weak and strong confounders with the full dataset. We do not run experiments with a small sample size since the STAR dataset already represents a small sample regime. These results corroborate previous findings that higher correlated confounders lead to greater power. 
 
\begin{figure*}[b]
\centering
    \begin{subfigure}[b]{0.3\textwidth}
        \centering
    \includegraphics[width=\textwidth]{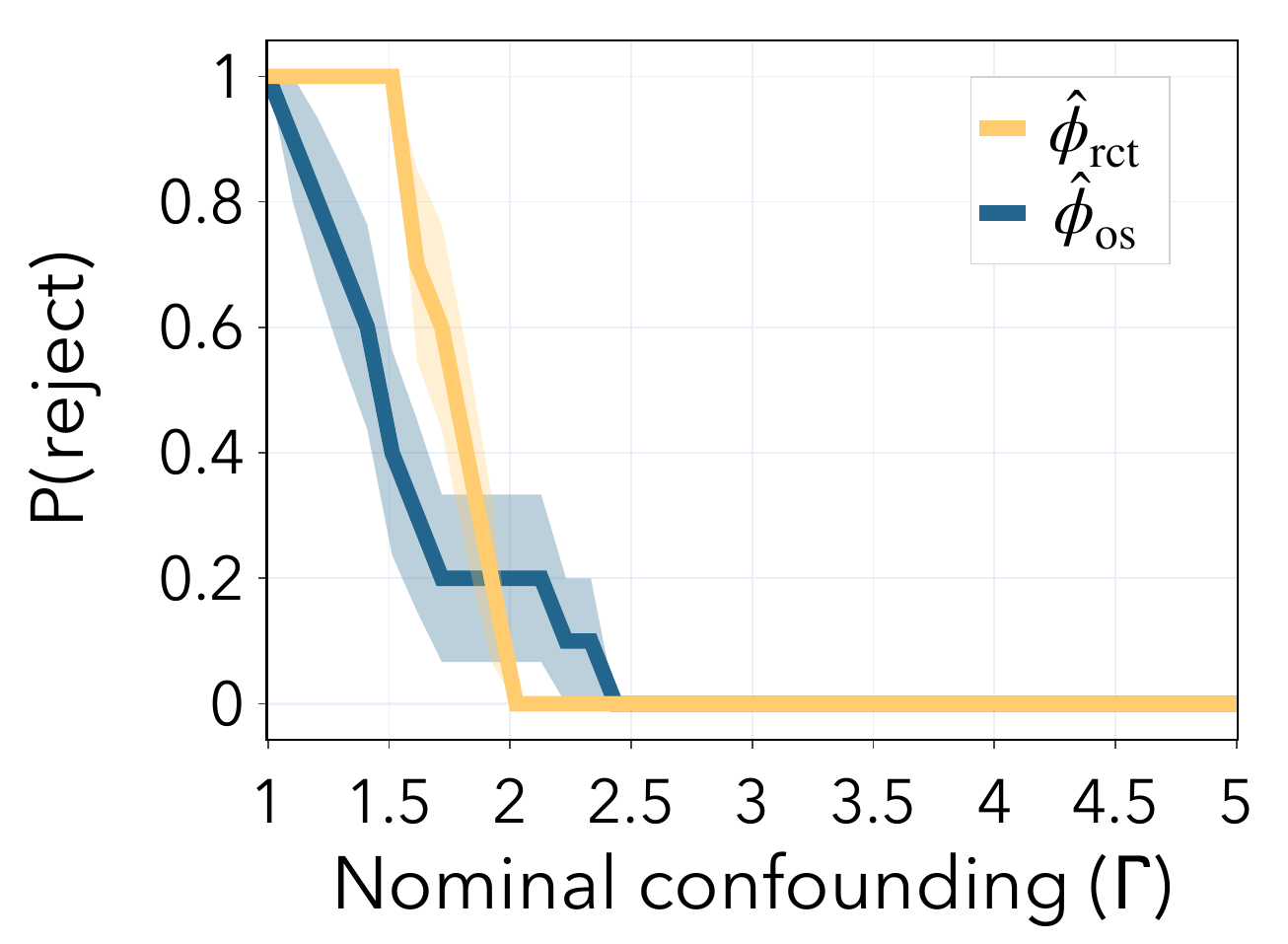}
        \caption{Weak confounder: ``free lunch"}
        \label{star:weak}
    \end{subfigure}%
    \begin{subfigure}[b]{0.3\textwidth}
        \centering
    \includegraphics[width=\textwidth]{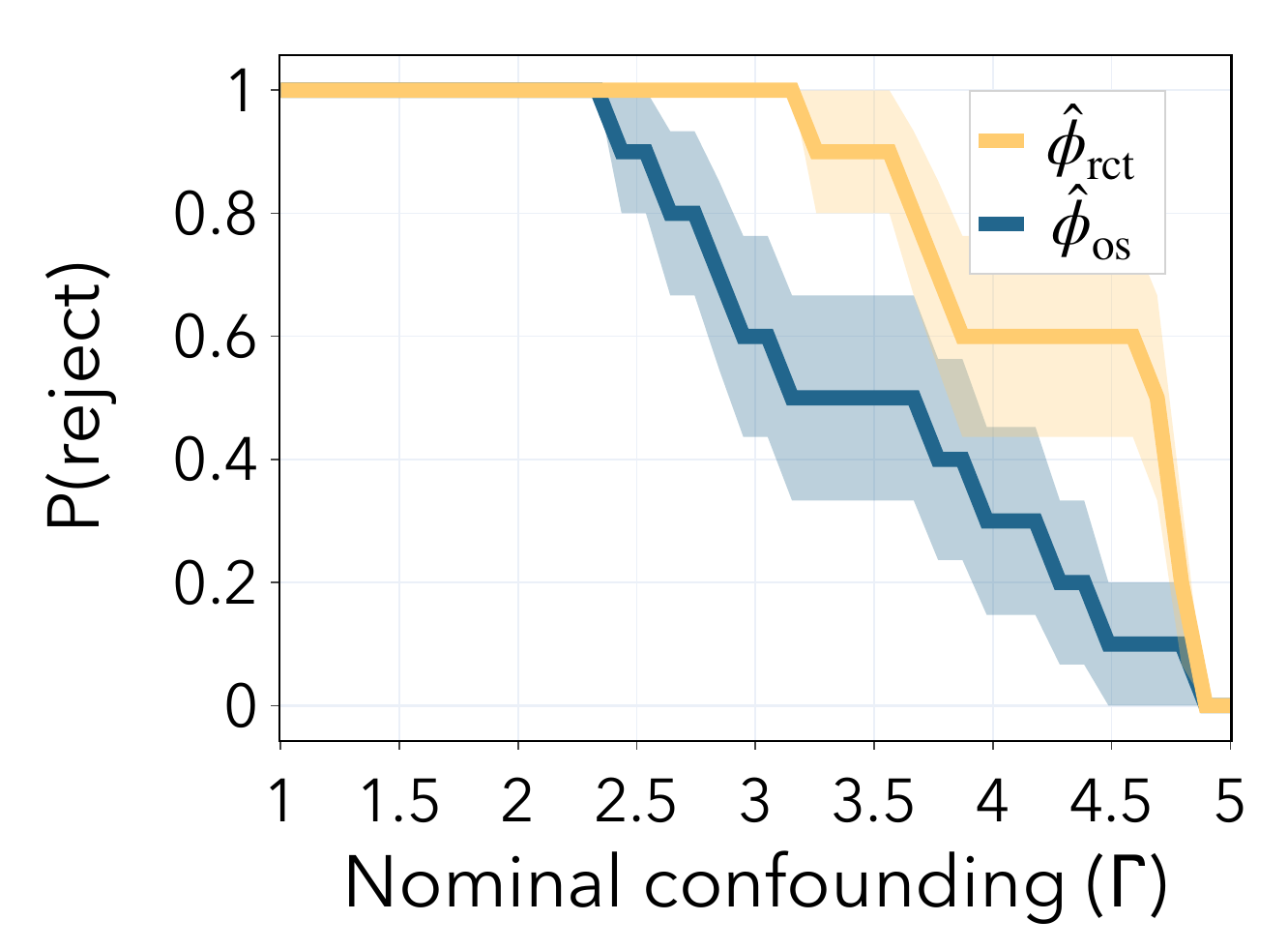}
        \caption{Strong confounder: $Y$}
        \label{star:strong}

    \end{subfigure}
    
    \caption{Probability of rejection for different choices of $\confvalue$ for the test for the STAR Project. For all the plots, the significance level is $\alpha=0.05$ and $\Gammatrue=5$. We use the original sample sizes $\nrct=600,\nobs=1.8K$. (a) weak confounder: ``free lunch" (b) strong confounder: outcome $Y$.
 }
    \label{fig:star_results}
\end{figure*}
\label{apx:additional_experiments}

\end{document}